\documentclass[11pt]{article}

\usepackage[utf8]{inputenc}
\usepackage[T1]{fontenc}
\usepackage{lmodern}

\usepackage{amsmath}
\usepackage{amssymb}
\usepackage{amsfonts}
\usepackage{amsthm}
\usepackage{bm}
\usepackage{mathtools}

\usepackage{graphicx}
\usepackage{subcaption}
\usepackage{booktabs}
\usepackage{multirow}
\usepackage{array}

\usepackage{xcolor}
\usepackage{enumitem}

\usepackage{algorithm}
\usepackage{algpseudocode}

\usepackage{url}
\usepackage{hyperref}
\usepackage[nameinlink,noabbrev]{cleveref}

\usepackage[numbers]{natbib}

\usepackage[margin=1in]{geometry}

\usepackage{microtype}

\newtheorem{theorem}{Theorem}

\theoremstyle{definition}

\theoremstyle{remark}

\title{
Understanding Submodular Information Measure Based Objectives for Representation Learning:\
A Variance and Separation Perspective
}

\usepackage{authblk}

\author[1]{Rishabh Iyer}
\author[1]{Truong Pham}
\author[2]{Anay Majee}

\affil[1]{The University of Texas at Dallas, Richardson, TX, USA}
\affil[2]{Adobe, San Jose, CA, USA}

\date{}

\begin{document}

\maketitle

\begin{abstract}

Submodular Information Measures (SIMs) based objectives and loss functions have recently emerged as a powerful framework for representation learning and multi-modal learning. In particular, the SCORE framework~\cite{majee2024score} demonstrated that SIMs can serve as effective objectives for supervised contrastive representation learning, yielding strong performance across balanced and long-tailed recognition settings. Despite these empirical successes, the geometric and statistical properties induced by different submodular information measures remain poorly understood. In particular, it is unclear what notions of compactness, separation, diversity, and coverage are optimized by different SIM formulations.

In this work, we develop a unified theoretical framework connecting submodular information measures to classical concepts in representation learning and statistical pattern recognition. We show that Total Information (TI) formulations based on diversity-oriented submodular functions naturally recover measures of intra-class variability. In particular, Graph Cut based TI objectives correspond to classical within-class variance, while LogDet based TI objectives recover generalized variance and covariance volume. We further show that Mutual Information (MI) formulations capture complementary notions of inter-class structure. Graph Cut based MI objectives are closely related to centroid separation and Fisher-style discriminative criteria, while Facility Location based MI objectives quantify representational overlap and modal coverage between classes. LogDet based MI objectives capture covariance-aware separation related to Mahalanobis and generalized discriminant distances.

To validate these characterizations, we design a suite of controlled synthetic experiments that independently manipulate variance, covariance structure, class imbalance, class separation, and multimodal overlap. Across all settings, the empirical behavior of the corresponding information measures closely matches the proposed theoretical interpretations. Our analysis reveals that different submodular information measures induce fundamentally different geometric biases, ranging from variance reduction and covariance control to coverage and multimodal representation. These results provide the first unified understanding of SIMs for representation learning and offer principled guidance for selecting and designing future submodular objectives.

\end{abstract}

\section{Introduction}

Learning effective representations is a central problem in modern machine learning. Deep neural networks trained with cross-entropy objectives have achieved remarkable success across vision, language, and multimodal tasks \citep{celoss,resnet,vgg,alexnet}. More recently, contrastive and metric-learning approaches such as SimCLR \citep{simclr2020}, MoCo \citep{moco}, Supervised Contrastive Learning \citep{supcon2020}, ArcFace \citep{arcface}, CosFace \citep{cosface}, Triplet Loss \citep{triplet}, and Center Loss \citep{centerloss} have demonstrated that explicitly shaping the geometry of the representation space can substantially improve generalization, robustness, and transferability. Despite their diversity, these methods share a common objective: learning embeddings that simultaneously encourage intra-class compactness and inter-class separation.

The importance of balancing compactness and separation can be traced back to classical discriminant analysis \citep{fisher1936}, where the quality of a representation is characterized by the ratio between between-class and within-class scatter. Modern representation learning methods instantiate this principle through a variety of objectives, including pairwise similarity losses \citep{triplet,lsl,n_pairs}, angular-margin losses \citep{arcface,cosface}, contrastive objectives \citep{supcon2020,simclr2020}, covariance regularization \citep{bardes2021vicreg}, and neighborhood-preservation criteria \citep{snn}. While these objectives differ significantly in their formulation, they all seek to control the geometry of representations through notions of variance, covariance, separation, and diversity.

In parallel, a distinct line of work has explored submodular optimization as a principled framework for modeling representation, diversity, coverage, and information \citep{wei15_subset,fujishige2005submodular,iyer2015submodular, kaushal2019demystifying, Kaushal_2019}. Submodular functions have been widely applied to data subset selection, summarization, active learning, and representation learning \citep{submod_diversity,kaushal2019demystifying,similar,prism,orient}. A particularly powerful family of objectives is provided by submodular information measures (SIMs), which generalize classical notions of mutual information through submodular set functions \citep{alt_gen_submod}. SIMs have enabled targeted data selection \citep{similar,talisman}, distribution-shift-aware subset selection \citep{orient}, and more recently representation learning and multi-modal learning \citep{majee2024score, majee2024smile}.

SCORE \citep{majee2024score} demonstrated that submodular information measures can serve as effective objectives for supervised representation learning. By constructing positive and negative sets from class labels and maximizing appropriate information measures between them, SCORE achieved competitive performance across multiple representation-learning benchmarks. SMILE~\cite{majee2024smile} extended the formulation of SCORE to few-shot representation learning, and showed that using the submodular mutual information based objectives can enable learning of new classes without forgetting known classes. In~\cite{majee2026looking}, the authors proposed a framework called CROWD that used the objectives from SCORE and SMILE for open world learning problems, while \cite{majee2026shasam} proposed a new family of loss functions for hard sample mining. Finally, \cite{pham2026sma} proposed the \emph{Submodular Modality Aligner}, a new family of objective functions for multi-modal alignment. 
However, despite its empirical success, a fundamental question remains unanswered: \emph{what geometric properties of the embedding space are induced by different submodular information measures?}.  In SCORE~\cite{majee2024score}, SMILE~\cite{majee2024smile}, SHaSaM~\cite{majee2026shasam}, and SMA~\cite{pham2026sma}, various variants of Total Submodular Information (TI) and Submodular Mutual Information (SMI) variants have been proposed -- including Facility Location TI, Log-Det TI, Graph Cut Mutual Information (GCMI), Facility Location Mutual Information (FLQMI), LogDet Mutual Information and their variants \citep{prism,orient,alt_gen_submod}. However, there is currently little understanding of the representation-learning biases encoded by these objectives.

This lack of understanding creates an important gap between theory and practice. Existing work largely treats different SIMs as interchangeable objective functions, selecting them empirically based on downstream performance~\cite{majee2024score, majee2024smile, majee2026shasam, pham2026sma}. Yet our experiments reveal that different SIMs induce fundamentally different geometric structures. For example, Graph Cut based objectives are closely related to variance-separation criteria and favor global class separation, while Facility Location based objectives emphasize coverage and multimodal representation. LogDet-based objectives, in contrast, explicitly capture covariance structure and generalized variance. These differences become especially pronounced in settings involving multimodal classes, long-tailed distributions, and heterogeneous semantic concepts, where distinct SIMs exhibit markedly different behaviors.

In this paper, we develop the first systematic theoretical and empirical study of submodular information measures for representation learning. We establish explicit connections between several widely used SIMs and classical statistical quantities including variance, covariance, generalized variance, class separation, and coverage. We show that Graph Cut based measures recover variance-separation objectives analogous to Fisher-style discriminant criteria, while Facility Location based measures provide coverage-oriented objectives that better capture multimodal class structure. We further characterize the regimes in which these objectives agree and diverge, and derive theoretical results explaining their behavior under class imbalance and multimodal distributions.

Our analysis yields both theoretical insights and practical guidance. Through synthetic studies, visualization experiments, and representation-learning benchmarks, we demonstrate that no single SIM is universally optimal. Instead, different information measures encode different inductive biases that make them suitable for different representation-learning regimes. These results provide a principled foundation for selecting and designing submodular information measures for representation learning and establish new connections between submodular optimization, information theory, and modern representation learning.

\paragraph{Contributions.}
Our primary contributions are:

\begin{enumerate}

\item \textbf{A unified geometric theory of Submodular Information Measures.}
We develop a theoretical framework connecting Submodular Information Measures (SIMs) to classical concepts in statistical pattern recognition and representation learning. Our analysis reveals that different SIMs recover fundamentally different notions of intra-class variability and inter-class separation.

\item \textbf{Characterization of Total Information objectives.}
We show that Graph Cut Total Information recovers classical within-class variance, LogDet Total Information recovers generalized variance through covariance volume, and Facility Location Total Information induces imbalance-aware separation margins that naturally emphasize rare and confusable classes.

\item \textbf{Characterization of Mutual Information objectives.}
We establish that Graph Cut Mutual Information is intimately connected to aggregate centroid separation and show that the combined GC-TI and GCMI objective exactly recovers the classical aggregate mean-separation criterion. We further show that LogDet Mutual Information recovers Mahalanobis-style covariance-aware separation, while Facility Location Mutual Information captures nearest-mode representational overlap in multimodal settings.

\item \textbf{A unified view of SCORE objectives.}
Our results provide the first theoretical explanation of the representation-learning biases induced by the SCORE framework, revealing how different combinations of Total Information and Mutual Information objectives recover distinct classical discriminative criteria.

\item \textbf{Comprehensive empirical validation.}
We design a suite of controlled synthetic experiments that independently manipulate variance, covariance structure, class imbalance, centroid separation, and multimodal overlap. Across all settings, the empirical behavior of the corresponding information measures closely matches the proposed theoretical characterizations.

\end{enumerate}

%
\section{Background and Preliminaries}

In this section, we briefly review representation learning, submodular functions, submodular information measures, and the SCORE framework that motivates our analysis.

\subsection{Representation Learning}

Let $\mathcal{D}={(x_i,y_i)}_{i=1}^{N}$ denote a labeled dataset with $K$ classes, where $y_i \in {1,\ldots,K}$. Given an encoder $f*\theta$, each sample is mapped to a representation

\begin{equation}
z_i = f_\theta(x_i) \in \mathbb{R}^{d}.
\end{equation}

A central objective in supervised representation learning is to learn embeddings that are both compact within a class and well separated across classes. Classical approaches such as Fisher's Linear Discriminant Analysis \citep{fisher1936} formalize this principle through within-class and between-class scatter matrices, while modern metric-learning and contrastive-learning methods achieve similar goals through pairwise, triplet, or contrastive objectives \citep{triplet,centerloss,simclr2020,supcon2020,arcface,cosface}.

For a class $C_c$, let

\begin{equation}
\mu_c = \frac{1}{|C_c|}\sum_{i\in C_c} z_i
\end{equation}

denote the class centroid. Throughout the paper, we will connect submodular information measures to classical notions of:

\begin{itemize}
\item \textbf{Intra-class variability}, which measures how dispersed samples are within a class.
\item \textbf{Inter-class separation}, which measures how distinct different classes are from one another.
\item \textbf{Coverage and representation}, which quantify how well one set of samples represents another.
\end{itemize}

\subsection{Submodular Functions}

Let $V$ denote a finite ground set. A set function $f:2^V \rightarrow \mathbb{R}$ is \emph{submodular} if it satisfies the diminishing returns property: $f(A \cup {j}) - f(A) \ge f(B \cup {j}) - f(B)$, for all $A \subseteq B \subseteq V$ and $j \notin B$. Submodular functions naturally model notions of representation, diversity, and coverage and have been extensively used in subset selection, summarization, active learning, and information maximization \citep{fujishige2005submodular,iyer2015submodular,submod_diversity}. Throughout this work, we focus on four commonly used submodular functions (Graph Cut, Facility Location, and Log Determinant):

\paragraph{Graph Cut (GC).}
Graph Cut captures a tradeoff between representation and diversity and is defined over a similarity matrix $S$ as

\begin{equation}
f_{\mathrm{GC}}(A) = 
\lambda \sum_{i\in A}\sum_{j\in V}s_{ij} - 
\sum_{i,j\in A}s_{ij}.
\end{equation}

\paragraph{Facility Location (FL).}
Facility Location is a representation-oriented function that measures how well a subset represents the ground set:

\begin{equation}
f_{\mathrm{FL}}(A) = \sum_{i\in V}
\max_{j\in A}
s_{ij}.
\end{equation}

\paragraph{LogDet (LD).}
LogDet is a diversity-oriented function based on the determinant of a similarity kernel. It is closely related to determinantal point processes and generalized variance \citep{Kulesza_2012}.

\begin{equation}
f_{\mathrm{LD}}(A) = \log \det(S_A).
\end{equation}

These functions represent distinct notions of diversity, coverage, and representation, and form the basis of the information measures~\cite{alt_gen_submod} studied in this paper.

\subsection{Submodular Information Measures}

Submodular Information Measures (SIMs) generalize classical information-theoretic quantities using submodular set functions \citep{alt_gen_submod}. Given a submodular function $f$, the corresponding submodular mutual information between two sets $A$ and $B$ is defined as

\begin{equation}
I_f(A;B) = f(A) + f(B) - f(A \cup B).
\end{equation}

Different choices of $f$ induce different notions of similarity and dependence. For example, Graph Cut Mutual Information (GCMI), Facility Location Mutual Information (FLQMI), and LogDet Mutual Information (LogDetMI) arise from choosing Graph Cut, Facility Location, and LogDet functions respectively \citep{similar,prism,orient}.

Complementary to MI is the notion of Total Information (TI), which measures the aggregate dependence among multiple sets:

\begin{equation}
TI_f(C_1,\ldots,C_K) = \sum_{c=1}^{K} f(C_c)
\end{equation}

Throughout this work, we analyze mainly the TI and MI formulations and show that they capture distinct geometric properties of the representation space.

\subsection{SCORE: Submodular Combinatorial Representation Learning}

Recently, Majee et al.~\citep{majee2024score} proposed SCORE, a supervised representation learning framework based on submodular information measures. SCORE constructs positive and negative sets using class labels and optimizes submodular objectives directly in the embedding space.

Given an embedding function $f_\theta$, SCORE defines a similarity matrix over representations and computes the Total Information (TI) and Total Correlation (TC) variants. Different choices of the underlying submodular function lead to different learning objectives, including Graph Cut, Facility Location, and LogDet variants.

Empirically, SCORE demonstrated that submodular information measures provide effective objectives for supervised representation learning and achieve strong performance across balanced, imbalanced, and long-tailed recognition settings. However, despite these empirical successes, the geometric and statistical properties induced by different information measures remain poorly understood.

The goal of this paper is to provide a unified theoretical understanding of these objectives. In particular, we ask:

\begin{quote}
What notions of variability, separation, and coverage are optimized by different submodular information measures?
\end{quote}

By connecting SIMs to classical statistical quantities such as variance, covariance volume, mean separation, and modal coverage, we develop a principled understanding of the representation-learning biases induced by different submodular objectives.

\section{A Unified Geometric View of Submodular Information Measures}

The central goal of this paper is to understand the geometric biases induced by different submodular information measures. While prior work has demonstrated that Graph Cut, Facility Location, LogDet, and Coverage-based information measures are effective for representation learning \citep{similar,prism,orient,majee2024score}, there is currently little understanding of the statistical quantities they optimize.

In this section, we introduce the key geometric quantities studied throughout the paper, present a generic formulation of SCORE, and summarize the main theoretical results established in subsequent sections.

\subsection{Variance, Separation, and Coverage}
Consider a collection of classes ${C_1,\ldots,C_K}$, where each class $C_c$ consists of embedded samples
$z_i \in \mathbb{R}^d$. Let $\frac{1}{|C_c|}
\sum_{i\in C_c}
z_i$ denote the centroid of class $C_c$, and let $\frac{1}{|C_c|}
\sum_{i\in C_c}
(z_i-\mu_c)(z_i-\mu_c)^T$ denote its covariance matrix. A central theme of this paper is that there is no unique notion of either intra-class variability or inter-class separation. Instead, different objectives emphasize different geometric properties of the representation space. We therefore begin by introducing several classical quantities that will later be connected to specific submodular information measures.

\subsubsection{Quantifying Intra-Class Structure}
The most common notion of intra-class variability is the within-class variance: 

\begin{align}
    \text{Intra-Class-Var}(C_1, \cdots, C_k) = \sum_{c=1}^{K}
\sum_{i\in C_c}
|z_i-\mu_c|^2 = \frac{1}{2}
\sum_{c=1}^{K}
\frac{1}{|C_c|}
\sum_{i,j\in C_c}
|z_i-z_j|^2
\end{align}
The second equality holds using standard variance identities.  Thus, within-class variance may be interpreted either through distances to class centroids or through average pairwise distances within a class.

A more general notion of variability is provided by the generalized variance:
\begin{align}
    \text{Intra-Class-GenVar}(C_1, \cdots, C_k) = \sum_{c=1}^{K}
\log \det(\Sigma_c)
\end{align} 
which captures not only the magnitude of variation but also the covariance structure and effective dimensionality of each class.

\subsubsection{Quantifying Inter-Class Structure}

\paragraph{Mean-Based Inter-Class Separation.}
A classical notion of inter-class separation is the pairwise centroid separation:
\begin{align}
\text{Mean-Separation}_{MB}(C_1, \cdots, C_k) = \sum_{i<j}
|\mu_i-\mu_j|^2
\end{align}
This quantity forms the basis of between-class scatter in Fisher's discriminant analysis and measures how far apart class centroids are in the embedding space.

\paragraph{Covariance-Aware Inter-Class Separation.}

A second notion of separation accounts for covariance structure through Mahalanobis distances:
\begin{align}
\text{Cov-Separation}(C_1, \cdots, C_k) = \sum_{i<j}
(\mu_i-\mu_j)^T
\Sigma^{-1}
(\mu_i-\mu_j)
\end{align}
where $\Sigma$ denotes an appropriate pooled covariance matrix. Unlike centroid separation, this quantity measures separation relative to the variability of the classes and therefore captures discriminability rather than distance alone.

\paragraph{Representational Overlap.}

A third notion of inter-class structure is representational rather than geometric. Instead of asking whether class centroids are far apart, one may ask whether the modes of one class can be represented by another class.

If overlap occurs across many modes, the two classes exhibit high representational overlap. If overlap is confined to only a few modes, the classes remain distinguishable despite potentially similar centroid locations.

As we show later, Facility Location based mutual information naturally captures this notion of representational overlap.

\subsection{A Generic SCORE Objective}

A central observation underlying SCORE~\citep{majee2024score} is that effective representation learning requires simultaneously controlling both intra-class variability and inter-class bias. The former determines the compactness and diversity of individual classes, while the latter governs the relationships among different classes.

This perspective naturally motivates the use of submodular information measures within a representation learning framework. Let $\mathcal{C}
=
\{C_1,\ldots,C_K\}$ denote the collection of classes and let $I_f(\cdot;\cdot)$ denote a submodular mutual information measure induced by a submodular function $f$.

A generic SCORE-style objective can be written as

\begin{equation}
\label{eq:score_generic}
\mathcal{L}_{\mathrm{SCORE}}
=
\lambda_{\mathrm{TI}}
TI_f(C_1,\ldots,C_K)
+
\lambda_{\mathrm{SMI}}
\sum_{i<j}
I_f(C_i;C_j),
\end{equation}
where $\lambda_{\mathrm{TI}}$ and $\lambda_{\mathrm{SMI}}$ control the relative importance of the two terms.
In the original SCORE paper~\cite{majee2024score}, the first term (TI) was proposed while the second term was used for few-shot representation learning in~\cite{majee2024smile}. In this paper, we analyze the combined objective, and we can get back the objective from \cite{majee2024score} by setting $\lambda_{\mathrm{SMI}}$ as 0.

This decomposition naturally separates representation learning into two components:

\begin{equation}
\text{SCORE}
=
\underbrace{\text{TI}}_{\text{Multi-Class Information}}
+
\underbrace{\text{SMI}}_{\text{Pairwise Class Information}}.
\label{eq:score_decomp}
\end{equation}
A key observation of this paper is that neither TI nor SMI admits a universal geometric interpretation. Rather, the quantities captured by these objectives depend fundamentally on the choice of the underlying submodular function.

Broadly speaking, submodular functions used in representation learning can be categorized according to whether they emphasize representation, diversity, or a combination of both.

\begin{itemize}

\item \textbf{Representation-oriented functions} emphasize how well one set represents another. Examples include Facility Location and Saturated Coverage.

\item \textbf{Diversity-oriented functions} emphasize dispersion among elements. Examples include LogDet and dispersion-based objectives.

\item \textbf{Hybrid functions} simultaneously contain representation and diversity components. The most prominent example is Graph Cut,

\begin{equation}
f_{\mathrm{GC}}(A)
=
\lambda
\sum_{i\in A}
\sum_{j\in V}
s_{ij}
-
\sum_{i,j\in A}
s_{ij},
\end{equation}

where the first term is representation-oriented while the second term is diversity-oriented.

\end{itemize}

This distinction is important because the geometric interpretation of TI and SMI depends strongly on the underlying submodular function. For diversity-oriented functions such as Graph Cut and LogDet, the TI formulations will be shown to recover classical notions of intra-class variability, including pairwise variance and generalized variance. In contrast, for representation-oriented functions such as Facility Location and Saturated Coverage, the TI formulations induce quantities that are closely related to inter-class representation and bias.

Similarly, different SMI functions induce different notions of inter-class structure. As we show later, Graph Cut Mutual Information is closely related to centroid separation, LogDet Mutual Information captures covariance-aware separation, while Facility Location Mutual Information captures representational overlap between classes.

Consequently, the geometric behavior of a SCORE objective is determined jointly by two factors:

\begin{enumerate}
    \item The information measure being optimized (TI or SMI).
    \item The representation-versus-diversity characteristics of the underlying submodular function.
\end{enumerate}

The primary goal of this paper is to characterize the notions of intra-class variability and inter-class bias induced by different combinations of these two factors.

\subsection{Summary of Main Results}
The central question studied in this paper is:

\begin{quote}
What notions of intra-class variability and inter-class bias are induced by different submodular information measures?
\end{quote}

Our analysis reveals that the behavior of a SCORE objective is governed jointly by the information measure being optimized (TI or SMI) and the representation-versus-diversity characteristics of the underlying submodular function. In particular, we establish the following results:
\begin{itemize}
\item \textbf{Graph Cut TI and intra-class variance.}
We show that Graph Cut based Total Information admits a decomposition into inter-class bias and intra-class variance terms. Under suitable parameterizations, the diversity component recovers classical within-class variance and its equivalent pairwise-distance formulation.

\item \textbf{LogDet TI and generalized variance.}
We show that LogDet based Total Information recovers generalized variance and covariance volume, thereby capturing both the magnitude and covariance structure of within-class variability.

\item \textbf{Representation-oriented TI objectives and inter-class bias.}
Unlike diversity-based objectives, representation-oriented functions such as Facility Location and Saturated Coverage induce Total Information formulations that depend explicitly on interactions between classes, yielding notions of inter-class representation and bias rather than classical intra-class variability.

\item \textbf{Graph Cut Mutual Information and mean separation.}
We establish a direct connection between Graph Cut Mutual Information and classical notions of centroid separation and between-class scatter.

\item \textbf{LogDet Mutual Information and covariance-aware separation.}
We show that LogDet Mutual Information captures covariance-aware notions of separation and admits connections to Mahalanobis and Fisher-style discriminative criteria.

\item \textbf{Facility Location Mutual Information and representational overlap.}
We show that Facility Location based mutual information induces a distinct notion of representational overlap, quantifying the extent to which the modes of one class are represented by another.
\end{itemize}

Together, these results provide a unified geometric interpretation of submodular information measures and explain the distinct representation-learning biases induced by different SCORE objectives. Table~\ref{tab:summary-tabls} provides a clean summary of the different objectives (TI and MI) and the respective variance and bias quantities recovered. 

\begin{table}[t]
\centering
\caption{Summary of the geometric biases induced by different Submodular Information Measures.}
\label{tab:summary-tabls}
\small
\begin{tabular}{lll}
\toprule
Measure & Category & Quantity Recovered \\
\midrule
GC-TI & Intra-Class & Within-class variance \\
LogDet-TI & Intra-Class & Generalized variance \\
FL-TI & Inter/Intra-Class & Rare-class margin expansion \\
\midrule
GCMI & Inter-Class & Aggregate overlap / centroid separation \\
GC-TI + GCMI & Inter + Intra Class & Aggregate mean separation \\
LogDetMI & Inter-Class & Mahalanobis separation \\
FLQMI & Inter-Class & Nearest-mode overlap \\
\bottomrule
\end{tabular}
\end{table}
\section{Total Information Objectives}
\label{sec:ti}

In this section, we study Total Information (TI) objectives and characterize the geometric quantities induced by different submodular functions. Recall that for a collection of classes $\mathcal{C} = 
{C_1,\ldots,C_K}$, the Total Information induced by a submodular function $f$ is given by

\begin{equation}
TI_f(C_1,\ldots,C_K) = 
\sum_{c=1}^{K}
f(C_c).
\end{equation}

Unlike Submodular Mutual Information, which explicitly models interactions between pairs of classes, Total Information aggregates the contribution of each class independently. Nevertheless, as we show below, the resulting objective need not be purely intra-class in nature. Depending on the underlying submodular function, TI objectives can induce both intra-class variability terms and inter-class bias terms.

A key observation of this paper is that the behavior of a TI objective depends strongly on whether the underlying submodular function is diversity-oriented, representation-oriented, or a combination of both. Diversity-based functions naturally recover classical notions of intra-class variability, while representation-based functions induce interactions between classes and therefore capture notions of inter-class representation and bias.

In the remainder of this section, we analyze four widely-used submodular functions: (i) Graph Cut (GC), (ii) LogDet (LD), (iii) Facility Location (FL), and (iv) Saturated Coverage (SatCov).

For each function, we derive the corresponding Total Information objective and connect it to classical statistical quantities arising in representation learning and pattern recognition.

\subsection{Graph Cut Total Information}
\label{sec:gcti}

We begin with Graph Cut, which occupies a unique position among submodular functions because it simultaneously contains representation and diversity components.

Recall that the Graph Cut function is defined as

\begin{equation}
f_{\mathrm{GC}}(A)
=
\lambda
\sum_{i\in A}
\sum_{j\in V}
s_{ij}
-
\sum_{i,j\in A}
s_{ij},
\end{equation}

where $S=[s_{ij}]$ denotes a similarity matrix and $\lambda \in [0,1]$ controls the balance between representation and diversity.

Substituting this expression into the Total Information objective yields

\begin{align}
TI_{\mathrm{GC}}
=
\sum_{c=1}^{K}
f_{\mathrm{GC}}(C_c) \
=
\lambda
\sum_{c=1}^{K}
\sum_{k\in C_c}
\sum_{l\in V}
s_{kl}
-
\sum_{c=1}^{K}
\sum_{k,l\in C_c}
s_{kl}.
\label{eq:gcti_basic}
\end{align}

Separating the first term into within-class and cross-class contributions gives

\begin{align}
TI_{\mathrm{GC}}
&=
\lambda
\sum_{c=1}^{K}
\sum_{k\in C_c}
\left(
\sum_{l\in C_c}
s_{kl}
+
\sum_{l\in V\setminus C_c}
s_{kl}
\right)
-
\sum_{c=1}^{K}
\sum_{k,l\in C_c}
s_{kl} \\
&=
\lambda
\sum_{c=1}^{K}
\sum_{k\in C_c}
\sum_{l\in V\setminus C_c}
s_{kl}
-
(1-\lambda)
\sum_{c=1}^{K}
\sum_{k,l\in C_c}
s_{kl}.
\label{eq:gcti_decomposition}
\end{align}

Equation~(\ref{eq:gcti_decomposition}) reveals an important decomposition. The first term measures interactions between a class and all points outside the class and therefore induces a notion of inter-class bias. The second term depends only on pairwise similarities within each class and therefore captures intra-class variability.

This decomposition is particularly interesting because it mirrors the classical representation learning objective of simultaneously maximizing inter-class separation while minimizing intra-class variability. Unlike purely diversity-oriented or purely representation-oriented functions, Graph Cut naturally contains both components within a single objective.

In the next subsection, we establish a direct connection between the diversity component of Graph Cut Total Information and classical within-class variance under an appropriate similarity model.

\subsubsection{Graph Cut TI and Within-Class Variance}

We now establish a connection between the diversity component of Graph Cut Total Information and classical within-class variance.

Consider the shifted Euclidean similarity

\begin{equation}
s_{ij} = M - |z_i-z_j|^2,
\label{eq:shifted_kernel}
\end{equation}

where $M$ is a sufficiently large constant ensuring non-negativity of the similarity matrix.

Substituting Equation~(\ref{eq:shifted_kernel}) into the diversity component of Equation~(\ref{eq:gcti_decomposition}) yields

\begin{align}
\sum_{k,l\in C_c}
s_{kl}
=
\sum_{k,l\in C_c}
\left(
M-|z_k-z_l|^2
\right) \
=
|C_c|^2 M - 
\sum_{k,l\in C_c}
|z_k-z_l|^2.
\end{align}

Thus, up to an additive constant, maximizing the diversity component of Graph Cut TI is equivalent to maximizing the total pairwise distance within each class.

The following theorem establishes the precise connection.

\begin{theorem}
\label{thm:gcti_variance}
Define the similarity kernel as follows: 
\begin{equation}
s_{ij} = M-|z_i-z_j|^2.
\end{equation}
The Graph-Cut Total Information with $\lambda = 0$ satisfies:

\begin{equation}
\text{TI}_{GC}(C_1, \cdots, C_k) = -
\sum_{c=1}^{K}
\sum_{k,l\in C_c}
s_{kl}
= 
-2
\sum_{c=1}^{K}
|C_c|
\sum_{k\in C_c}
|z_k-\mu_c|^2
+
\mathrm{const},
\end{equation}

where $\mu_c$ denotes the centroid of class $C_c$.
Consequently, Graph Cut Total Information recovers classical within-class variance up to a class-dependent scaling factor and an additive constant.
\end{theorem}

Theorem~\ref{thm:gcti_variance} provides a direct connection between Graph Cut based Total Information and one of the most fundamental quantities in statistical pattern recognition. In particular, the diversity component of Graph Cut TI recovers the classical within-class scatter objective used in Fisher's Linear Discriminant Analysis.

An important implication is that minimizing Graph Cut TI encourages compact class representations by reducing within-class variability. Furthermore, because the derivation is based on pairwise distances, the resulting objective remains valid for multimodal classes and does not rely on any explicit centroid computation.

In Section~\ref{sec:experiments}, we empirically validate this characterization by constructing synthetic datasets with controlled within-class variance and demonstrating that Graph Cut Total Information tracks the resulting variability almost perfectly.

\subsection{LogDet Total Information and Generalized Variance}
\label{sec:logdet_ti}

We next consider LogDet, a diversity-oriented submodular function that has been widely used in subset selection, determinantal point processes, and diversity maximization.

Recall that the LogDet function is defined as

\begin{equation}
f_{\mathrm{LD}}(A)
=
\log \det(K_A),
\end{equation}

where $K_A$ denotes the principal kernel matrix corresponding to the subset $A$. The corresponding Total Information objective is

\begin{equation}
TI_{\mathrm{LD}}(C_1, \cdots, C_k)
=
\sum_{c=1}^{K}
\log \det(K_{C_c}).
\label{eq:logdet_ti}
\end{equation}

Unlike Graph Cut, which decomposes into representation and diversity terms, LogDet is purely diversity-oriented. Consequently, we expect LogDet based Total Information to capture a notion of intra-class variability. To formalize this connection, let

\begin{equation}
\Sigma_c
=
\frac{1}{|C_c|}
\sum_{i\in C_c}
(z_i-\mu_c)(z_i-\mu_c)^T
\end{equation}
denote the covariance matrix of class $C_c$.
A classical measure of variability is the generalized variance

\begin{equation}
GV(C_c)
=
\log \det(\Sigma_c),
\end{equation}

which measures the covariance volume occupied by the class.

The following theorem establishes a connection between LogDet Total Information and generalized variance.

\begin{theorem}
\label{thm:logdet_variance}

Assume that the samples within each class are drawn from a distribution with covariance matrix $\Sigma_c$, and let $K_{C_c}$ denote the corresponding kernel matrix constructed from these samples. Under standard regularity conditions and in the large-sample regime,

\begin{equation}
\log \det(K_{C_c})
= a_c + b_c \log \det(\Sigma_c)
+ o(1),
\end{equation}
where $a_c$ and $b_c>0$ are constants depending on the kernel and embedding dimension.

Consequently,

\begin{equation}
TI_{\mathrm{LD}} = \sum_{c=1}^{K}
\log \det(K_{C_c})
\end{equation}
is a monotone transformation of
\begin{equation}
\sum_{c=1}^{K}
\log \det(\Sigma_c),
\end{equation}

and therefore recovers the generalized variance of the class distributions.
\end{theorem}

Theorem~\ref{thm:logdet_variance} reveals that LogDet Total Information captures a richer notion of variability than Graph Cut. While Graph Cut is related to within-class variance through pairwise distances, LogDet additionally incorporates covariance structure and effective dimensionality.
In particular, two classes may have identical trace variance but very different covariance volumes. Graph Cut treats such classes similarly, whereas LogDet distinguishes them through the determinant of their covariance matrices.
Consequently, LogDet TI may be viewed as a covariance-aware generalization of classical within-class variance.

In Section~\ref{sec:experiments}, we validate this interpretation through synthetic experiments that independently vary covariance volume while controlling other factors. We observe an almost perfect monotonic relationship between LogDet Total Information and generalized variance, supporting the theoretical characterization above.

\subsection{Facility Location Total Information and Rare Class Margins}
\label{sec:fl_ti}

We next consider Facility Location, one of the most widely used representation-oriented submodular functions.

Recall that the Facility Location function is defined as

\begin{equation}
f_{\mathrm{FL}}(A)
=
\sum_{i\in V}
\max_{j\in A}
s_{ij},
\end{equation}

where $V$ denotes the ground set and $S=[s_{ij}]$ is a similarity matrix.

The corresponding Total Information objective is

\begin{equation}
TI_{\mathrm{FL}}
=
\sum_{c=1}^{K}
f_{\mathrm{FL}}(C_c)
=
\sum_{c=1}^{K}
\sum_{i\in V}
\max_{j\in C_c}
s_{ij}.
\label{eq:fl_ti}
\end{equation}

Unlike Graph Cut and LogDet, Facility Location contains no explicit diversity component. Instead, it measures how well each class represents the points in the ground set through nearest-neighbor similarity.

To understand the structure of Equation~(\ref{eq:fl_ti}), we separate the contribution of points belonging to the class from those lying outside the class:

\begin{align}
TI_{\mathrm{FL}}
=
\sum_{c=1}^{K}
\Bigg[
\sum_{i\in C_c}
\max_{j\in C_c}
s_{ij}
+
\sum_{i\in V\setminus C_c}
\max_{j\in C_c}
s_{ij}
\Bigg].
\label{eq:fl_ti_split}
\end{align}

Assuming normalized similarities satisfying $s_{ii}=1$, the first term is constant since each point is maximally represented by itself. Consequently,

\begin{equation}
TI_{\mathrm{FL}}
=
N
+
\sum_{c=1}^{K}
\sum_{i\in V\setminus C_c}
\max_{j\in C_c}
s_{ij},
\label{eq:fl_ti_cross}
\end{equation}

where $N$ denotes the total number of samples.

Equation~(\ref{eq:fl_ti_cross}) reveals that Facility Location Total Information is entirely governed by cross-class interactions. Unlike Graph Cut and LogDet, which recover classical notions of intra-class variability, Facility Location TI measures how well each class represents the remainder of the dataset.

To make this dependence explicit, define

\begin{equation}
R_c
=
\frac{1}{N-n_c}
\sum_{i\in V\setminus C_c}
\max_{j\in C_c}
s_{ij},
\label{eq:fl_representation}
\end{equation}

where $n_c=|C_c|$ denotes the size of class $C_c$. The quantity $R_c$ measures the average representational affinity between class $C_c$ and all samples outside the class. Substituting Equation~(\ref{eq:fl_representation}) into Equation~(\ref{eq:fl_ti_cross}) gives

\begin{equation}
TI_{\mathrm{FL}}
=
N
+
\sum_{c=1}^{K}
(N-n_c)R_c.
\label{eq:fl_weighted}
\end{equation}
Consequently, the contribution of class $C_c$ is weighted by $N-n_c$. Classes with fewer samples therefore receive larger relative weight than classes with many samples. As a result, Facility Location Total Information naturally emphasizes rare and underrepresented classes.

\begin{theorem}[Facility Location TI Induces Larger Margins for Rare Classes]
\label{thm:flti_rare_confusable}

Assume normalized similarities satisfying $s_{ii}=1$ and suppose each class $C_c$ forms a compact cluster and let
$\delta_c$ denote the distance between the centroid of class $C_c$
and its nearest competing class. Further assume an RBF similarity

\begin{equation}
s_{ij}
=
\exp
\left(
-\frac{\|z_i-z_j\|^2}{2\tau^2}
\right).
\end{equation}

Then the cross-class contribution of class $C_c$ satisfies

\begin{equation}
\sum_{i\in V\setminus C_c}
\max_{j\in C_c}
s_{ij}
=
\Theta
\!\left(
(N-n_c)
\exp
\left(
-\frac{\delta_c^2}{2\tau^2}
\right)
\right).
\end{equation}

Consequently, the influence of class $C_c$ on Facility Location
Total Information scales as

\begin{equation}
\mathrm{Influence}(C_c)
=
\Theta
\!\left(
(N-n_c)
\exp
\left(
-\frac{\delta_c^2}{2\tau^2}
\right)
\right).
\end{equation}

Furthermore, if two classes $C_a$ and $C_b$ contribute equally to the objective, then

\begin{equation}
\delta_a^2-\delta_b^2
=
2\tau^2
\log
\left(
\frac{N-n_a}{N-n_b}
\right).
\end{equation}

In particular, if $n_a<n_b$, then $\delta_a>\delta_b$. Thus, for the same effective contribution to the FL-TI objective, smaller classes require larger separation margins.
\end{theorem}

\paragraph{Interpretation.}

Theorem~\ref{thm:flti_rare_confusable} provides a theoretical explanation for the strong empirical performance of Facility Location based SCORE objectives in long-tailed settings \citep{majee2024score}. Unlike variance-based objectives, whose contribution typically scales with the number of samples within a class, Facility Location Total Information penalizes class overlap with a weight proportional to the number of samples outside the class.

As a result, tail classes incur a larger penalty when they overlap with neighboring classes. To reduce the objective, the representation must therefore increase the separation of rare classes more aggressively than that of head classes. The induced margin grows approximately as

\begin{equation}
\delta_c^2
=
2\tau^2
\log(N-n_c)
+
\mathrm{const}.
\end{equation}

showing that Facility Location naturally allocates larger margins to underrepresented classes. This result provides a principled explanation for the long-tail robustness observed in SCORE-FL~\cite{majee2024score} and suggests that representation-oriented submodular objectives implicitly perform class-dependent margin adaptation without requiring explicit reweighting or resampling.

\section{Submodular Mutual Information Objectives}
\label{sec:smi}

In the previous section, we studied Total Information objectives and showed that different submodular functions induce different notions of intra-class variability and inter-class bias. We now turn our attention to Submodular Mutual Information (SMI) objectives.

Recall that for two sets $A$ and $B$, the submodular mutual information induced by a submodular function $f$ is

\begin{equation}
I_f(A;B)
=
f(A)
+
f(B)
-
f(A \cup B).
\end{equation}

Unlike Total Information, which aggregates information within classes, SMI explicitly measures interactions between sets. Consequently, SMI objectives are naturally suited for modeling inter-class structure. A generic SCORE objective based on pairwise mutual information takes the form

\begin{equation}
\sum_{i<j}
I_f(C_i;C_j),
\end{equation}

where $C_i$ and $C_j$ denote different classes. The central question studied in this section is:

\begin{quote}
What notions of inter-class bias and separation are induced by different submodular mutual information measures?
\end{quote}

We answer this question for Graph Cut, LogDet, and Facility Location based mutual information and establish connections to classical notions of mean separation, covariance-aware discrimination, and representational overlap.

\subsection{Graph Cut Mutual Information and Mean Separation}
\label{sec:gcmi}

We now turn to Graph Cut Mutual Information (GCMI), which provides a complementary perspective to the Graph Cut Total Information objective studied in Section~\ref{sec:gcti}.

While Graph Cut Total Information was shown to recover within-class variance, Graph Cut Mutual Information captures interactions between classes and therefore provides a natural measure of inter-class structure. For two classes $C_a$ and $C_b$, Graph Cut Mutual Information takes the form

\begin{equation}
I_{\mathrm{GC}}(C_a;C_b)
=
\sum_{i\in C_a}
\sum_{j\in C_b}
s_{ij},
\label{eq:gcmi}
\end{equation}

which measures the aggregate similarity between samples belonging to different classes.
Intuitively, if two classes are well separated, cross-class similarities are small and the corresponding Graph Cut Mutual Information is low. Conversely, if two classes overlap substantially, many cross-class similarities become large, leading to a larger GCMI value.

We now establish a precise connection between GCMI and classical notions of class separation.
Consider two classes $C_a$ and $C_b$ with centroids $\mu_a$ and $\mu_b$. Define the mean separation

\begin{equation}
D(C_a,C_b)
=
\|\mu_a-\mu_b\|^2.
\end{equation}

The following theorem shows that mean separation admits a decomposition into within-class variance terms and a Graph Cut Mutual Information term.

\begin{theorem}[GC Decomposition of Aggregate Mean Separation]
\label{thm:aggregate_gcmi_separation}

Let
\begin{equation}
s_{ij}
=
M-\|z_i-z_j\|^2,
\end{equation}
where $M$ is a sufficiently large constant, and let $n_c=|C_c|$.
Define the aggregate centroid-separation objective
\begin{equation}
\mathcal{D}_{\mathrm{mean}}
=
\sum_{a<b}
n_a n_b
\|\mu_a-\mu_b\|^2.
\end{equation}

Then minimizing the negative mean-separation objective is equivalent,
up to positive scaling and additive constants, to minimizing
\begin{equation}
-\mathcal{D}_{\mathrm{mean}}
\equiv
\alpha_{\mathrm{TI}}
\widetilde{TI}_{\mathrm{GC}}^{\lambda=0}
+
\alpha_{\mathrm{MI}}
\sum_{a<b}
I_{\mathrm{GC}}(C_a;C_b),
\end{equation}
where
\begin{equation}
\widetilde{TI}_{\mathrm{GC}}^{\lambda=0}
=
\sum_{c=1}^{K}
\frac{N-n_c}{2n_c}
\left(
-\sum_{i,j\in C_c}
s_{ij}
\right)
\end{equation}
is a class-size weighted Graph Cut Total Information term, and
$\alpha_{\mathrm{TI}},\alpha_{\mathrm{MI}}>0$ are constants depending
only on normalization.

In the balanced-class setting, 
$\widetilde{TI}_{\mathrm{GC}}^{\lambda=0}$ reduces to a positive scalar
multiple of the standard $TI_{\mathrm{GC}}^{\lambda=0}$. Thus, in the
balanced case, the GC-based SCORE objective consisting of standard
Graph Cut Total Information with $\lambda=0$ and pairwise GCMI recovers
the classical objective of maximizing aggregate inter-class centroid
separation.
\end{theorem}

\paragraph{Interpretation.}
Theorem~\ref{thm:aggregate_gcmi_separation} shows that the natural
multi-class separation objective
\(\sum_{a<b} n_a n_b \|\mu_a-\mu_b\|^2\)
is exactly decomposed into two SCORE components: a TI term over classes
and a pairwise SMI term over class pairs. Thus, the combined GC-based
SCORE objective recovers the classical discriminative principle of
minimizing within-class scatter while maximizing between-class
separation.

\subsection{LogDet Mutual Information and Mahalanobis Separation}
\label{sec:logdetmi}

While Graph Cut Mutual Information captures centroid separation, it is
insensitive to the covariance structure of the underlying classes. In
many representation learning settings, classes may exhibit highly
anisotropic or correlated distributions, making covariance-aware notions
of separation more appropriate.

To address this limitation, we consider LogDet Mutual Information
(LogDetMI). Recall that

\begin{equation}
I_{\mathrm{LD}}(C_a;C_b)
=
f_{\mathrm{LD}}(C_a)
+
f_{\mathrm{LD}}(C_b)
-
f_{\mathrm{LD}}(C_a \cup C_b),
\end{equation}

where

\begin{equation}
f_{\mathrm{LD}}(A)
=
\log \det(K_A),
\end{equation}

and $K_A$ denotes the kernel matrix associated with subset $A$.

Unlike Graph Cut Mutual Information, which depends primarily on
cross-class similarities, LogDetMI incorporates the covariance
structure of the participating classes through the determinants of the
corresponding kernel matrices.

The following theorem establishes a connection between LogDetMI and
covariance-aware notions of class separation.

\begin{theorem}[LogDetMI and Mahalanobis Separation]
\label{thm:logdetmi_mahalanobis}

Let classes $C_a$ and $C_b$ have means $\mu_a,\mu_b$, covariance
matrices $\Sigma_a,\Sigma_b$, and class proportions

\begin{equation}
p=\frac{n_a}{n_a+n_b},
\qquad
q=\frac{n_b}{n_a+n_b}.
\end{equation}

Define the pooled covariance

\begin{equation}
\Sigma_w
=
p\Sigma_a+q\Sigma_b
\end{equation}

and the Mahalanobis separation

\begin{equation}
\mathcal{M}_{ab}
=
(\mu_a-\mu_b)^T
\Sigma_w^{-1}
(\mu_a-\mu_b).
\end{equation}

Under the covariance LogDet approximation, LogDet Mutual Information satisfies

\begin{equation}
I_{\mathrm{LD}}(C_a;C_b)
=
\mathrm{const}
-
\log
\left(
1+
pq\mathcal{M}_{ab}
\right).
\end{equation}

Consequently, LogDetMI is a monotone decreasing function of
Mahalanobis separation. Thus, minimizing LogDetMI is equivalent, up to
monotone transformations, to maximizing covariance-normalized class
separation.
\end{theorem}

\paragraph{Interpretation.}

Theorem~\ref{thm:logdetmi_mahalanobis} reveals a fundamental distinction
between Graph Cut and LogDet based mutual information. While Graph Cut
Mutual Information is primarily governed by centroid separation, LogDetMI
depends on the Mahalanobis separation between classes and therefore
explicitly accounts for covariance structure and anisotropy.

As a result, two pairs of classes with identical Euclidean centroid
separation may receive substantially different LogDetMI values if their
covariance structures differ. In particular, LogDetMI naturally favors
separation along directions of low within-class variance, closely
mirroring classical Fisher and Mahalanobis discriminative criteria. Consequently, LogDetMI may be viewed as a covariance-aware
generalization of centroid separation, providing a richer notion of
inter-class discrimination than Graph Cut Mutual Information.

In Section~\ref{sec:experiments}, we validate this interpretation by
independently varying centroid separation and covariance structure. The
resulting experiments demonstrate a strong monotonic relationship
between LogDetMI and Mahalanobis separation, supporting the theoretical
characterization above.

\subsection{Facility Location Mutual Information and Nearest Mode Overlap}
\label{sec:flqmi}

While Graph Cut and LogDet based mutual information capture different
notions of class separation, Facility Location Mutual Information
(FLQMI) captures a fundamentally different property: representational
overlap. Recall that the Facility Location function is defined as

\begin{equation}
f_{\mathrm{FL}}(A)
=
\sum_{i\in V}
\max_{j\in A}
s_{ij}.
\end{equation}
The corresponding Facility Location Mutual Information between two
classes $C_a$ and $C_b$ is:
\begin{equation}
I_{\mathrm{FL}}(C_a;C_b)
=
f_{\mathrm{FL}}(C_a)
+
f_{\mathrm{FL}}(C_b)
-
f_{\mathrm{FL}}(C_a \cup C_b).
\end{equation}
The following theorem shows the connection between FLQMI and nearest-mode overlap. 

\begin{theorem}[FLQMI and Nearest-Mode Overlap]
\label{thm:flqmi_modes}

Assume classes $C_a$ and $C_b$ consist of compact modes with centers
$\{\nu_{a,r}\}_{r=1}^{m_a}$ and
$\{\nu_{b,s}\}_{s=1}^{m_b}$, respectively, and let

\begin{equation}
d_{rs}
=
\|\nu_{a,r}-\nu_{b,s}\|
\end{equation}

denote the distance between mode centers.

Further assume an RBF similarity

\begin{equation}
s_{ij}
=
\exp
\left(
-\frac{\|z_i-z_j\|^2}{2\tau^2}
\right).
\end{equation}

Then Facility Location Mutual Information admits the approximation

\begin{equation}
I_{\mathrm{FL}}(C_a;C_b)
\approx
\sum_{r,s}
w_{rs}
\exp
\left(
-\frac{d_{rs}^2}{2\tau^2}
\right),
\end{equation}

where

\begin{equation}
w_{rs}
=
n_{a,r}\mathbf{1}\{s=s^*(r)\}
+
n_{b,s}\mathbf{1}\{r=r^*(s)\},
\end{equation}

and $s^*(r)$ and $r^*(s)$ denote the nearest neighboring modes of
mode $r$ and mode $s$, respectively.

Consequently, FLQMI is primarily governed by nearest-mode overlap rather
than aggregate overlap across all mode pairs.
\end{theorem}

\paragraph{Interpretation.}

Theorem~\ref{thm:flqmi_modes} reveals a fundamental distinction between
Graph Cut and Facility Location based mutual information. While Graph
Cut Mutual Information aggregates contributions from all pairs of modes,
Facility Location Mutual Information is dominated by nearest-mode
interactions. Consequently, GCMI measures aggregate overlap mass between classes,
whereas FLQMI measures localized representational overlap. In
multimodal settings, FLQMI is therefore more sensitive to rare or
poorly separated modes and behaves as a soft nearest-mode separation
criterion. This explains why two pairs of classes may exhibit similar aggregate
overlap yet receive substantially different FLQMI values depending on
how their modes are arranged in the representation space.
\section{Experiments}
\label{sec:experiments}

In this section, we empirically validate the theoretical characterizations developed in Sections~\ref{sec:ti} and~\ref{sec:smi}. Our goal is not to benchmark representation-learning performance, but rather to verify that the proposed information measures exhibit the geometric behavior predicted by the theory.

To this end, we design a collection of controlled synthetic experiments in which individual geometric properties of the representation space can be manipulated independently. Specifically, we construct datasets that vary:

\begin{itemize}
\item within-class variance,
\item covariance structure and anisotropy,
\item class imbalance,
\item centroid separation, and
\item multimodal overlap.
\end{itemize}

For each experiment, we compute the corresponding submodular information measure and compare it against the classical quantity predicted by the theory. A strong monotonic relationship between the two provides empirical support for the theoretical characterization.

Unless otherwise specified, all synthetic datasets are generated from Gaussian mixtures in a $d$-dimensional embedding space. Similarities are computed using either the shifted Euclidean kernel

\begin{equation}
s_{ij}
=
M-|z_i-z_j|^2,
\end{equation}

or the Gaussian RBF kernel

\begin{equation}
s_{ij}
=
\exp
\left(
-\frac{|z_i-z_j|^2}{2\tau^2}
\right),
\end{equation}

depending on the assumptions of the corresponding theorem.

We evaluate the following theoretical predictions:

\begin{enumerate}
\item Graph Cut Total Information recovers within-class variance.
\item LogDet Total Information recovers generalized variance.
\item Graph Cut Total Information and Graph Cut Mutual Information jointly recover aggregate mean separation.
\item Facility Location Total Information induces larger margins for rare classes.
\item LogDet Mutual Information recovers Mahalanobis separation.
\item Facility Location Mutual Information recovers nearest-mode overlap.
\end{enumerate}

We begin by validating the connection between Graph Cut Total Information and classical within-class variance.

\subsection{Experiment 1: Graph Cut Total Information and Within-Class Variance}
\label{sec:exp1}

Theorem~\ref{thm:gcti_variance} states that Graph Cut Total Information with $\lambda=0$ recovers the classical within-class variance when the similarity function is chosen as a shifted Euclidean kernel, $s_{ij}
=
M-|z_i-z_j|^2$.  To validate this result, we generate synthetic Gaussian classes while varying the within-class standard deviation $\sigma$. For each dataset, we compute:

\begin{enumerate}
\item the classical within-class variance,
\item the equivalent pairwise-distance formulation,
\item the quantity recovered by Graph Cut Total Information using the shifted Euclidean kernel, and
\item a Graph Cut objective computed using an RBF kernel.
\end{enumerate}
Table~\ref{tab:gcti_variance} summarizes the results.
\begin{table}[t]
\centering
\caption{Graph Cut Total Information recovers within-class variance exactly under the shifted Euclidean kernel.}
\label{tab:gcti_variance}
\small
\begin{tabular}{ccccc}
\toprule
$\sigma$ &
StdVar &
Pairwise Identity &
GC Shift &
GC RBF \\
\midrule
0.05 & 9.85 & 9.85 & 9.85 & 9.82 \\
0.10 & 40.94 & 40.94 & 40.94 & 40.31 \\
0.20 & 154.51 & 154.51 & 154.51 & 145.86 \\
0.40 & 642.86 & 642.86 & 642.86 & 511.20 \\
0.80 & 2531.28 & 2531.28 & 2531.28 & 1187.79 \\
1.20 & 5617.50 & 5617.50 & 5617.50 & 1474.99 \\
1.60 & 9858.35 & 9858.35 & 9858.35 & 1550.94 \\
2.00 & 15831.65 & 15831.65 & 15831.65 & 1573.80 \\
\bottomrule
\end{tabular}
\end{table}
Several observations follow. First, the classical variance, pairwise-distance identity, and Graph Cut Total Information recovered using the shifted Euclidean kernel are numerically identical across all settings. This provides direct empirical verification of Theorem~\ref{thm:gcti_variance}.

Second, the RBF version exhibits a strong monotonic relationship with variance but does not recover the variance exactly. This behavior is expected since the RBF kernel introduces a nonlinear transformation of pairwise distances. Nevertheless, larger within-class variance consistently produces larger Graph Cut values, indicating that the variance interpretation extends qualitatively beyond the shifted Euclidean setting.

Overall, these results confirm that Graph Cut Total Information recovers classical within-class variance and therefore behaves as a variance-minimization objective in representation learning.

\subsection{Experiment 2: LogDet Total Information and Generalized Variance}
\label{sec:exp2}
Theorem~\ref{thm:logdet_variance} suggests that LogDet Total Information recovers the generalized variance of the class distributions. Unlike Graph Cut Total Information, which measures classical within-class variance, LogDet Total Information incorporates the full covariance structure of the embeddings through the determinant of the covariance matrix.

To validate this result, we generate Gaussian classes while varying the covariance scale parameter $\sigma$. For each dataset, we compute:

\begin{enumerate}
\item the aggregate covariance log-determinant,
\begin{equation}
\sum_c \log \det(\Sigma_c),
\end{equation}

\item the corresponding LogDet Total Information objective, and

\item the classical centroid variance used in Experiment~\ref{sec:exp1}.
\end{enumerate}
Table~\ref{tab:logdet_variance} summarizes the result thaat LogDet TI closely tracks the generalized variance. The corresponding correlation matrix (between LogDet TI, the generalized variance, and the classical variance) is shown in Table~\ref{tab:logdet_corr}.

\begin{table}[t]
\centering
\caption{LogDet Total Information closely tracks generalized variance.}
\label{tab:logdet_variance}
\small
\begin{tabular}{cccc}
\toprule
$\sigma$ &
CovLogDet &
LogDet-TI &
CentroidVar \\
\midrule
0.05 & -303.21 & -2507.71 & 9.89 \\
0.10 & -233.78 & -2366.14 & 39.65 \\
0.20 & -164.16 & -1954.84 & 159.08 \\
0.40 & -95.29 & -1194.24 & 631.07 \\
0.80 & -26.27 & -373.00 & 2521.22 \\
1.20 & 14.57 & -95.03 & 5697.51 \\
1.60 & 42.65 & -23.05 & 10016.39 \\
2.00 & 65.80 & -4.94 & 15847.87 \\
\bottomrule
\end{tabular}
\end{table}

\begin{table}[t]
\centering
\caption{Correlation between generalized variance, LogDet Total Information, and classical variance.}
\label{tab:logdet_corr}
\small
\begin{tabular}{lccc}
\toprule
& CovLogDet & LogDet-TI & CentroidVar \\
\midrule
CovLogDet & 1.000 & 0.984 & 0.803 \\
LogDet-TI & 0.984 & 1.000 & 0.773 \\
CentroidVar & 0.803 & 0.773 & 1.000 \\
\bottomrule
\end{tabular}
\end{table}

Several observations follow. First, LogDet Total Information exhibits an almost perfect monotonic relationship with the covariance log-determinant, achieving a correlation of $0.984$. This provides strong empirical support for Theorem~\ref{thm:logdet_variance}. Second, while LogDet-TI remains correlated with classical variance, the relationship is substantially weaker. This is expected since the determinant captures covariance volume rather than merely the sum of marginal variances. Consequently, LogDet-TI incorporates information about anisotropy and covariance structure that is ignored by standard variance measures.

Overall, these results confirm that LogDet Total Information behaves as a generalized variance objective and captures richer covariance information than Graph Cut Total Information.

\subsection{Experiment 3: Facility Location Total Information Emphasizes Rare and Confusable Classes}
\label{sec:exp3}
Theorem~\ref{thm:flti_rare_confusable} predicts that Facility Location
Total Information assigns larger weight to classes that are both rare
and poorly separated from competing classes. To validate this prediction, we construct a long-tailed dataset
consisting of three classes:
a head class ($n=1000$),
a medium class ($n=300$),
and a tail class ($n=50$).
The tail class is positioned near the head class, making it highly
confusable, while the medium class is placed farther away. Table~\ref{tab:flti_longtail} reports the per-class contributions of
GC-TI, LogDet-TI, and FL-TI.

\begin{table}[t]
\centering
\caption{Per-class contribution shares on a long-tailed dataset.}
\label{tab:flti_longtail}
\small
\begin{tabular}{lccc}
\toprule
Class &
FL-TI &
GC-TI &
LogDet-TI \\
\midrule
Head ($n=1000$) & 5.5\% & 91.5\% & 32.7\% \\
Medium ($n=300$) & 1.0\% & 8.2\% & 33.4\% \\
Tail ($n=50$) & \textbf{93.5\%} & 0.2\% & 33.9\% \\
\bottomrule
\end{tabular}
\end{table}

Several observations follow.
\begin{itemize}
    \item First, GC-TI is dominated by the head class, reflecting its connection
to within-class variance and class cardinality.
\item Second, LogDet-TI distributes importance nearly uniformly across
classes, since all classes possess similar covariance structure.
\item In contrast, FL-TI assigns more than 93\% of its total contribution to
the rare and confusable tail class. This behavior is precisely
predicted by Theorem~\ref{thm:flti_rare_confusable}, which shows that
the influence of a class scales as
$(N-n_c)\exp(-\delta_c^2/(2\tau^2))$.
\end{itemize}
These results provide a theoretical explanation for the strong
performance of Facility Location based SCORE objectives in long-tailed
recognition settings. Unlike variance-based objectives, FL-TI naturally
prioritizes underrepresented classes without requiring explicit class
reweighting.

\subsection{Experiment 4: Aggregate Mean Separation and the GC-SCORE Objective}
\label{sec:exp4}

Theorem~\ref{thm:aggregate_gcmi_separation} establishes a direct connection between classical discriminative representation learning and Submodular Information Measures. In particular, it shows that the aggregate centroid-separation objective

\begin{equation}
\mathcal{D}_{\mathrm{mean}}
=
\sum_{a<b}
n_a n_b
|\mu_a-\mu_b|^2
\end{equation}

admits an exact decomposition into a weighted Graph Cut Total Information term and a pairwise Graph Cut Mutual Information term.

To validate this result, we generate synthetic datasets with increasing class separation and compute:

\begin{enumerate}
\item the aggregate mean-separation objective,
\item the weighted Graph Cut Total Information term,
\item the pairwise Graph Cut Mutual Information term, and
\item their sum predicted by Theorem~\ref{thm:aggregate_gcmi_separation}.
\end{enumerate}

Table~\ref{tab:gcmi_mean_sep} summarizes the results.

\begin{table}[t]
\centering
\caption{Verification of the exact decomposition of aggregate mean separation into GC-TI and GCMI components.}
\label{tab:gcmi_mean_sep}
\small
\begin{tabular}{ccccc}
\toprule
Mean Sep &
$-\mathcal{D}_{\mathrm{mean}}$ &
GC-TI &
GCMI &
GC-TI + GCMI \\
\midrule
1 & -79,303 & -753,619 & 674,315 & -79,303 \\
2 & -295,956 & -1,299,882 & 1,003,926 & -295,956 \\
4 & -1,225,091 & -2,676,234 & 1,451,144 & -1,225,091 \\
6 & -2,773,310 & -5,041,493 & 2,268,183 & -2,773,310 \\
8 & -4,906,642 & -7,798,381 & 2,891,739 & -4,906,642 \\
10 & -7,730,433 & -11,006,376 & 3,275,943 & -7,730,433 \\
\bottomrule
\end{tabular}
\end{table}

To quantify the quality of the decomposition, we compare the centered versions of the two objectives. The absolute reconstruction error is shown in Table~\ref{tab:gcmi_error}.

\begin{table}[t]
\centering
\caption{Reconstruction error of the GC-TI + GCMI decomposition.}
\label{tab:gcmi_error}
\small
\begin{tabular}{ccc}
\toprule
Mean Sep &
Centered Error \\
\midrule
1 & $0.0$ \
2 & $8.7\times10^{-11}$ \\
4 & $2.3\times10^{-10}$ \\
6 & $1.4\times10^{-9}$ \\
8 & $1.9\times10^{-9}$ \\
10 & $3.7\times10^{-9}$ \\
\bottomrule
\end{tabular}
\end{table}

The decomposition is exact up to numerical precision, with reconstruction errors on the order of $10^{-9}$. Furthermore, the correlation between the theoretical objective $-\mathcal{D}_{\mathrm{mean}}$ and the recovered GC-SCORE objective is equal to $1.0$.

These results provide direct empirical verification of Theorem~\ref{thm:aggregate_gcmi_separation}. More importantly, they reveal that the GC-based SCORE objective is not merely correlated with classical discriminative objectives but exactly recovers the aggregate centroid-separation criterion. Consequently, Graph Cut Total Information and Graph Cut Mutual Information together admit a precise geometric interpretation as a decomposition of classical mean-separation based representation learning.

\subsection{Experiment 5: LogDet Mutual Information and Mahalanobis Separation}
\label{sec:exp5}

Theorem~\ref{thm:logdetmi_mahalanobis} predicts that LogDet Mutual Information is a monotone decreasing function of Mahalanobis separation. Unlike Graph Cut Mutual Information, which depends primarily on centroid separation, LogDetMI additionally incorporates covariance structure and anisotropy through the pooled covariance matrix.

To validate this characterization, we perform two complementary experiments.

\paragraph{Experiment 5A: Varying Centroid Separation.}

We first generate two Gaussian classes and vary the distance between their centroids while keeping the covariance structure fixed. Here $\delta
=
\|\mu_a-\mu_b\|$. For each setting, we compute the covariance-based LogDetMI approximation and the corresponding Mahalanobis separation. Table~\ref{tab:logdetmi_sep} summarizes the results.

\begin{table}[t]
\centering
\caption{LogDetMI as a function of centroid separation.}
\label{tab:logdetmi_sep}
\small
\begin{tabular}{ccc}
\toprule
$\delta$ &
LogDetMI &
Mahalanobis Separation \\
\midrule
0.5 & -0.270 & 0.097 \\
1.0 & -0.320 & 0.307 \\
2.0 & -0.491 & 1.109 \\
3.0 & -0.719 & 2.420 \\
4.0 & -0.969 & 4.241 \\
6.0 & -1.457 & 9.411 \\
8.0 & -1.887 & 16.619 \\
\bottomrule
\end{tabular}
\end{table}

The resulting correlations are $\mathrm{corr}(\mathrm{LogDetMI},\mathcal{M}) = -0.968$ and $\mathrm{corr}(\delta,\mathrm{LogDetMI}) = -0.979$. These results confirm the strong monotonic relationship predicted by Theorem~\ref{thm:logdetmi_mahalanobis}.

\paragraph{Experiment 5B: Varying Covariance Orientation.}

Next, we fix the centroid separation while varying the relative orientation of the covariance ellipses. This experiment isolates the covariance-dependent component of Mahalanobis separation. Table~\ref{tab:logdetmi_angle} summarizes the results.

\begin{table}[t]
\centering
\caption{Effect of covariance orientation on LogDetMI.}
\label{tab:logdetmi_angle}
\small
\begin{tabular}{ccc}
\toprule
Angle &
LogDetMI &
Mahalanobis Separation \\
\midrule
$0^\circ$ & -0.673 & 2.369 \\
$15^\circ$ & -1.002 & 2.818 \\
$30^\circ$ & -1.596 & 3.582 \\
$45^\circ$ & -2.072 & 3.987 \\
$60^\circ$ & -2.393 & 4.178 \\
$90^\circ$ & -2.636 & 4.295 \\
\bottomrule
\end{tabular}
\end{table}
The correlation between LogDetMI and Mahalanobis separation in this setting is $\mathrm{corr}(\mathrm{LogDetMI},\mathcal{M}) = -0.906$. Notably, substantial changes in LogDetMI occur despite the centroid distance remaining fixed. This behavior cannot be explained by Euclidean separation alone and directly reflects the covariance-aware nature of LogDetMI.

Taken together, Experiments 5A and 5B provide strong empirical support for Theorem~\ref{thm:logdetmi_mahalanobis}. While Graph Cut Mutual Information captures centroid separation, LogDetMI captures Mahalanobis separation and therefore incorporates covariance structure, anisotropy, and orientation into the representation-learning objective.

\subsection{Experiment 6: FLQMI and Multimodal Representational Overlap}
\label{sec:exp_flqmi}

Theorem~\ref{thm:flqmi_modes} predicts that Facility Location Mutual
Information is governed by nearest-mode overlap rather than aggregate
cross-class similarity. To validate this characterization, we construct
a synthetic multimodal dataset in which aggregate overlap is held
constant while representational overlap is varied.

Class $A$ consists of four equally-sized modes. We consider two
configurations of class $B$:

\begin{enumerate}
\item \textbf{Concentrated Overlap}: the overlap between the two classes
is concentrated around a single mode of $A$.

\item \textbf{Distributed Overlap}: the same amount of overlap is
distributed across all four modes of $A$.
\end{enumerate}

By construction, the total cross-class similarity is nearly identical
between the two settings. However, the distributed configuration
provides representatives for a much larger fraction of the modes of
class $A$.

Table~\ref{tab:flqmi_multimodal} summarizes the results.

\begin{table}[t]
\centering
\caption{FLQMI distinguishes distributed and concentrated overlap despite nearly identical aggregate overlap.}
\label{tab:flqmi_multimodal}
\small
\begin{tabular}{lcccc}
\toprule
Setting &
GCMI &
FLQMI &
$A\!\rightarrow\!B$ Coverage &
$B\!\rightarrow\!A$ Coverage \\
\midrule
Concentrated &
3384.78 &
129.15 &
79.39 &
49.76 \\
Distributed &
3384.52 &
439.35 &
389.58 &
49.77 \\
\bottomrule
\end{tabular}
\end{table}

Several observations follow. First, GCMI remains essentially unchanged
across the two settings, indicating that aggregate overlap mass is
nearly identical. Consequently, Graph Cut Mutual Information is unable
to distinguish between concentrated and distributed overlap.

In contrast, FLQMI increases from $129.15$ to $439.35$, a factor of
approximately $3.4\times$. This increase mirrors the dramatic rise in
representational coverage from class $A$ to class $B$, while the
reverse coverage remains nearly constant.

These results provide strong empirical support for
Theorem~\ref{thm:flqmi_modes}. While GCMI measures aggregate overlap
across all mode pairs, FLQMI is sensitive to how that overlap is
distributed among the modes. Consequently, FLQMI behaves as a
nearest-mode or representational-overlap criterion, making it
particularly well-suited for multimodal representation learning.
\section{Conclusion}
\label{sec:conclusion}

Submodular Information Measures have recently emerged as powerful objectives for representation learning, subset selection, and information maximization. Despite their empirical success, the geometric and statistical properties induced by different information measures have remained poorly understood. In this work, we developed a unified theoretical framework connecting Submodular Information Measures to classical notions of variance, covariance, separation, and representational overlap.

Our analysis revealed that different information measures induce fundamentally different representation-learning biases. For Total Information objectives, we showed that Graph Cut Total Information recovers classical within-class variance, LogDet Total Information recovers generalized variance through covariance volume, and Facility Location Total Information naturally emphasizes rare and confusable classes by inducing imbalance-aware separation margins. For Mutual Information objectives, we established connections between Graph Cut Mutual Information and aggregate centroid separation, LogDet Mutual Information and Mahalanobis separation, and Facility Location Mutual Information and nearest-mode representational overlap.

Perhaps most notably, we showed that the combination of Graph Cut Total Information and Graph Cut Mutual Information admits an exact decomposition of the classical aggregate mean-separation objective. This result provides a direct bridge between Submodular Information Measures and classical discriminative representation learning, revealing that certain SCORE objectives recover well-studied statistical criteria exactly rather than approximately.

We further validated these theoretical predictions through a suite of controlled synthetic experiments. Across all settings, the empirical behavior of the corresponding information measures closely matched the proposed theoretical characterizations, providing strong evidence that the geometric interpretations developed in this work accurately describe the behavior of SIM-based objectives.

Beyond providing theoretical insight, our results offer practical guidance for selecting information measures in representation learning. Variance-oriented objectives such as Graph Cut are well suited for compactness and class separation, LogDet based objectives provide covariance-aware discrimination, and Facility Location based objectives naturally emphasize multimodal structure, rare classes, and representational coverage. We believe that these insights will facilitate the principled design of future submodular objectives and contribute to a deeper understanding of combinatorial approaches to representation learning.

An important direction for future work is extending these analyses beyond supervised representation learning to settings such as self-supervised learning, retrieval-augmented generation, and multimodal foundation models, where Submodular Information Measures have recently demonstrated strong empirical success. More broadly, we hope that the theoretical framework developed in this paper serves as a foundation for understanding and designing the next generation of representation learning, self-supervised learning, and multi-modal learning objectives guided by combinatorial objective functions.

\bibliographystyle{plainnat}
\bibliography{reference}

@inproceedings{lsl,
    Author = {Hyun Oh Song and Yu Xiang and Stefanie Jegelka and Silvio Savarese},
    Title = {Deep Metric Learning via Lifted Structured Feature Embedding},
    Booktitle = {Computer Vision and Pattern Recognition (CVPR)},
    Year = {2016}
}

@inproceedings{n_pairs,
 author = {Sohn, Kihyuk},
 booktitle = {Advances in Neural Inf. Processing Systems},
 title = {Improved Deep Metric Learning with Multi-class N-pair Loss Objective},
 year = {2016}
}

@inproceedings{supcon2020,
 author = {Khosla, Prannay and Teterwak, Piotr and Wang, Chen and Sarna, Aaron and Tian, Yonglong and Isola, Phillip and Maschinot, Aaron and Liu, Ce and Krishnan, Dilip},
 booktitle = {Advances in Neural Information Processing Systems},
 title = {Supervised Contrastive Learning},
 year = {2020}
}

@article{simclr2020,
  title={A Simple Framework for Contrastive Learning of Visual Representations},
  author={Chen, Ting and Kornblith, Simon and Norouzi, Mohammad and Hinton, Geoffrey},
  journal={Intl. Conf. on Machine Learning (ICML)},
  year={2020}
}

@inproceedings{snn,
  title={Analyzing and Improving Representations with the Soft Nearest Neighbor Loss},
  author={Nicholas Frosst and Nicolas Papernot and Geoffrey E. Hinton},
  booktitle={International Conference on Machine Learning},
  year={2019}
}

@ARTICLE{alt_gen_submod,
  author={Iyer, Rishabh and Khargonkar, Ninad and Bilmes, Jeff and Asnani, Himanshu},
  journal={IEEE Transactions on Information Theory}, 
  title={Generalized Submodular Information Measures: Theoretical Properties, Examples, Optimization Algorithms, and Applications}, 
  year={2022},
  volume={68},
  number={2},
  pages={752-781}
}

@phdthesis{iyer2015submodular,
  title={Submodular optimization and machine learning: Theoretical results, unifying and scalable algorithms, and applications},
  author={Iyer, Rishabh Krishnan},
  year={2015}
}

@inproceedings{kaushal2019demystifying,
  title={Demystifying multi-faceted video summarization: tradeoff between diversity, representation, coverage and importance},
  author={Kaushal, Vishal and Iyer, Rishabh and Doctor, Khoshrav and Sahoo, Anurag and Dubal, Pratik and Kothawade, Suraj and Mahadev, Rohan and Dargan, Kunal and Ramakrishnan, Ganesh},
  booktitle={2019 IEEE Winter Conference on Applications of Computer Vision (WACV)},
  pages={452--461},
  year={2019},
  organization={IEEE}
}

@inproceedings{prism,
  author       = {Suraj Kothawade and
                  Vishal Kaushal and
                  Ganesh Ramakrishnan and
                  Jeff A. Bilmes and
                  Rishabh K. Iyer},
  title        = {{PRISM:} {A} Rich Class of Parameterized Submodular Information Measures
                  for Guided Data Subset Selection},
  booktitle    = {Thirty-Sixth {AAAI} Conference on Artificial Intelligence, {AAAI}},
  pages        = {10238--10246},
  year         = {2022}
}

@inproceedings{talisman,
  author       = {Suraj Kothawade and
                  Saikat Ghosh and
                  Sumit Shekhar and
                  Yu Xiang and
                  Rishabh K. Iyer},
  title        = {Talisman: Targeted Active Learning for Object Detection with Rare
                  Classes and Slices Using Submodular Mutual Information},
  booktitle    = {Computer Vision - {ECCV} 2022 - 17th European Conference},
  year         = {2022}
}

@article{similar,
  title={{SIMILAR}: Submodular information measures based active learning in realistic scenarios},
  author={Kothawade, Suraj and Beck, Nathan and Killamsetty, Krishnateja and Iyer, Rishabh},
  journal={Advances in Neural Information Processing Systems},
  volume={34},
  year={2021}
}

@InProceedings{wei15_subset,
  title = 	 {Submodularity in Data Subset Selection and Active Learning},
  author = 	 {Wei, Kai and Iyer, Rishabh and Bilmes, Jeff},
  booktitle = 	 {ICML},
  year = 	 {2015}
}

@inproceedings{Kaushal_2019,
   title={Learning From Less Data: A Unified Data Subset Selection and Active Learning Framework for Computer Vision},
   booktitle={2019 IEEE Winter Conference on Applications of Computer Vision (WACV)},
   author={Kaushal, Vishal and Iyer, Rishabh and Kothawade, Suraj and Mahadev, Rohan and Doctor, Khoshrav and Ramakrishnan, Ganesh},
   year={2019}
}

@inproceedings{submod_diversity,
    title = "A Class of Submodular Functions for Document Summarization",
    author = "Lin, Hui  and
      Bilmes, Jeff",
    booktitle = "Proceedings of the 49th Annual Meeting of the Association for Computational Linguistics: Human Language Technologies",
    year = "2011",
}

@article{Kulesza_2012,
   title={Determinantal Point Processes for Machine Learning},
   volume={5},
   ISSN={1935-8245},
   number={2–3},
   journal={Foundations and Trends® in Machine Learning},
   author={Kulesza, Alex},
   year={2012},
   pages={123–286} 
}

@article{celoss,
  author = {Rumelhart, David E and Hinton, Geoffrey E and Williams, Ronald J},
  journal = {nature},
  number = 6088,
  pages = {533--536},
  title = {Learning representations by back-propagating errors},
  volume = 323,
  year = 1986
}

@book{fujishige2005submodular,
  title={Submodular functions and optimization},
  author={Fujishige, Satoru},
  year={2005},
  publisher={Elsevier}
}

@inproceedings{orient,
 author = {Karanam, Athresh and Killamsetty, Krishnateja and Kokel, Harsha and Iyer, Rishabh},
 booktitle = {Advances in Neural Information Processing Systems},
 title = {ORIENT: Submodular Mutual Information Measures for Data Subset Selection under Distribution Shift},
 volume = {35},
 year = {2022}
}

@inproceedings{resnet,
  author    = {He, Kaiming and Zhang, Xiangyu and Ren, Shaoqing and Sun, Jian},
  title     = {Deep Residual Learning for {Image} {Recognition}},
  booktitle = {IEEE Conf. on Computer Vision and Pattern Recognition (CVPR)},
  year      = {2016}
}

@inproceedings{vgg,
  author    = {Karen Simonyan and Andrew Zisserman},
  title     = {Very Deep Convolutional Networks for Large-Scale Image Recognition},
  booktitle = {Intl. Conf. on Learning Representations},
  year      = {2015}
}

@inproceedings{alexnet,
 author = {Krizhevsky, Alex and Sutskever, Ilya and Hinton, Geoffrey E},
 booktitle = {Advances in Neural Information Processing Systems},
 title = {ImageNet Classification with Deep Convolutional Neural Networks},
 year = {2012}
}

@INPROCEEDINGS{moco,
  author={He, Kaiming and Fan, Haoqi and Wu, Yuxin and Xie, Saining and Girshick, Ross},
  booktitle={2020 IEEE/CVF Conference on Computer Vision and Pattern Recognition (CVPR)}, 
  title={Momentum Contrast for Unsupervised Visual Representation Learning}, 
  year={2020}
}

@inproceedings{arcface,
    title={Arcface: Additive angular margin loss for deep face recognition},
    author={Deng, Jiankang and Guo, Jia and Xue, Niannan and Zafeiriou, Stefanos},
    booktitle={Proceedings of the IEEE/CVF Conference on Computer Vision and Pattern Recognition},
    pages={4690--4699},
    year={2019}
}

@InProceedings{cosface,
author = {Wang, Hao and Wang, Yitong and Zhou, Zheng and Ji, Xing and Gong, Dihong and Zhou, Jingchao and Li, Zhifeng and Liu, Wei},
title = {CosFace: Large Margin Cosine Loss for Deep Face Recognition},
booktitle = {Proceedings of the IEEE Conference on Computer Vision and Pattern Recognition (CVPR)},
month = {June},
year = {2018}
}

@InProceedings{centerloss,
author="Wen, Yandong
and Zhang, Kaipeng
and Li, Zhifeng
and Qiao, Yu",
editor="Leibe, Bastian
and Matas, Jiri
and Sebe, Nicu
and Welling, Max",
title="A Discriminative Feature Learning Approach for Deep Face Recognition",
booktitle="Computer Vision -- ECCV 2016",
year="2016",
pages="499--515"
}

@INPROCEEDINGS{triplet,
  author={Schroff, Florian and Kalenichenko, Dmitry and Philbin, James},
  booktitle={IEEE Conf. on Computer Vision and Pattern Recognition (CVPR)}, 
  title={FaceNet: A unified embedding for face recognition and clustering}, 
  year={2015}
}

@article{fisher1936,
  author={Fisher, Ronald A.},
  title={The use of multiple measurements in taxonomic problems},
  journal={Annals of Eugenics},
  year={1936}
}

@article{bardes2021vicreg,
  title={Vicreg: Variance-invariance-covariance regularization for self-supervised learning},
  author={Bardes, Adrien and Ponce, Jean and LeCun, Yann},
  journal={arXiv preprint arXiv:2105.04906},
  year={2021}
}

@article{majee2024score,
  title={Score: Submodular combinatorial representation learning},
  author={Majee, Anay and Kothawade, Suraj and Killamsetty, Krishnateja and Iyer, Rishabh},
  journal={In ICML},
  year={2024}
}

@inproceedings{majee2024smile,
  title={SMILe: Leveraging submodular mutual information for robust few-shot object detection},
  author={Majee, Anay and Sharp, Ryan and Iyer, Rishabh},
  booktitle={European Conference on Computer Vision},
  pages={350--366},
  year={2024},
  organization={Springer}
}

@article{mercer1909functions,
  title={Functions of positive and negative type, and their connection with the theory of integral equations},
  author={Mercer, James},
  journal={Philosophical Transactions of the Royal Society of London. Series A},
  volume={209},
  pages={415--446},
  year={1909}
}

@inproceedings{scholkopf1998nonlinear,
  title={Nonlinear Component Analysis as a Kernel Eigenvalue Problem},
  author={Sch{\"o}lkopf, Bernhard and Smola, Alexander and M{\"u}ller, Klaus-Robert},
  booktitle={Neural Computation},
  volume={10},
  number={5},
  pages={1299--1319},
  year={1998}
}

@article{steinwart2012mercer,
  title={Mercer's Theorem on General Domains: On the Interaction between Measures, Kernels, and RKHSs},
  author={Steinwart, Ingo and Scovel, Clint},
  journal={Constructive Approximation},
  volume={35},
  pages={363--417},
  year={2012}
}

@article{pham2026sma,
  title={SMA: Submodular Modality Aligner For Data Efficient Multimodal Learning},
  author={Pham, Truong and Majee, Anay and Iyer, Rishabh},
  journal={arXiv preprint arXiv:2605.12872},
  year={2026}
}

@article{majee2026shasam,
  title={SHaSaM: Submodular Hard Sample Mining for Fair Facial Attribute Recognition},
  author={Majee, Anay and Iyer, Rishabh},
  journal={In Proc. WACV},
  year={2026}
}

@article{majee2026looking,
  title={Looking Beyond the Known: Towards a Data Discovery Guided Open-World Object Detection},
  author={Majee, Anay and Gangrade, Amitesh and Iyer, Rishabh},
  journal={Advances in Neural Information Processing Systems},
  volume={38},
  pages={25050--25078},
  year={2026}
}

\appendix

\section{Proofs of Theoretical Results}
\label{app:proofs}

In this appendix, we provide proofs of the theoretical results presented in the main paper.

\subsection{Proof of Theorem~\ref{thm:gcti_variance}}
\label{app:proof_gcti_variance}
\begin{proof}
For $\lambda=0$, the Graph-Cut Total Information is

\begin{equation}
\mathrm{TI}_{\mathrm{GC}}(C_1,\ldots,C_K)
=
-
\sum_{c=1}^{K}
\sum_{k,l\in C_c}
s_{kl}.
\end{equation}

Substituting $s_{kl}=M-\|z_k-z_l\|^2$ gives

\begin{align}
\mathrm{TI}_{\mathrm{GC}}(C_1,\ldots,C_K)
&=
-
\sum_{c=1}^{K}
\sum_{k,l\in C_c}
\left(
M-\|z_k-z_l\|^2
\right) \\
&=
-
M\sum_{c=1}^{K}|C_c|^2
+
\sum_{c=1}^{K}
\sum_{k,l\in C_c}
\|z_k-z_l\|^2 .
\end{align}

The first term is independent of the representations and is therefore an additive constant. Hence,

\begin{equation}
\mathrm{TI}_{\mathrm{GC}}(C_1,\ldots,C_K)
=
\sum_{c=1}^{K}
\sum_{k,l\in C_c}
\|z_k-z_l\|^2
+
\mathrm{const}.
\end{equation}

Using the standard identity

\begin{equation}
\sum_{k,l\in C_c}
\|z_k-z_l\|^2
=
2|C_c|
\sum_{k\in C_c}
\|z_k-\mu_c\|^2,
\end{equation}

we obtain

\begin{equation}
\mathrm{TI}_{\mathrm{GC}}(C_1,\ldots,C_K)
=
2
\sum_{c=1}^{K}
|C_c|
\sum_{k\in C_c}
\|z_k-\mu_c\|^2
+
\mathrm{const}.
\end{equation}

Thus, Graph-Cut Total Information with $\lambda=0$ recovers the classical within-class variance up to a class-dependent scaling factor and an additive constant.
\end{proof}

\subsection{Proof of Theorem~\ref{thm:logdet_variance}}
\label{app:proof_logdet_variance}
\begin{proof}
We first prove the result for the covariance, or linear, kernel, and then
describe the extension to general positive-definite kernels.

Let $Z_c \in \mathbb{R}^{n_c \times d}$ denote the centered data matrix
for class $C_c$, whose rows are $z_i-\mu_c$. The empirical covariance is

\begin{equation}
\widehat{\Sigma}_c
=
\frac{1}{n_c}
Z_c^T Z_c .
\end{equation}

For the linear kernel,

\begin{equation}
K_{C_c}
=
Z_c Z_c^T .
\end{equation}

Since $Z_cZ_c^T$ and $Z_c^TZ_c$ have the same nonzero eigenvalues, the
nonzero determinant, or pseudo-determinant, satisfies

\begin{align}
\log \operatorname{pdet}(K_{C_c})
&=
\log \operatorname{pdet}(Z_cZ_c^T) \\
&=
\log \det(Z_c^TZ_c) \\
&=
\log \det(n_c \widehat{\Sigma}_c) \\
&=
d\log n_c
+
\log \det(\widehat{\Sigma}_c).
\end{align}

Under standard regularity assumptions, the empirical covariance
$\widehat{\Sigma}_c$ converges to the population covariance
$\Sigma_c$ as $n_c\to\infty$. Therefore,

\begin{equation}
\log \operatorname{pdet}(K_{C_c})
=
d\log n_c
+
\log \det(\Sigma_c)
+
o(1).
\end{equation}

Thus, for the linear kernel, the claim holds with

\begin{equation}
a_c=d\log n_c,
\qquad
b_c=1.
\end{equation}

For more general positive-definite kernels, the kernel matrix
$K_{C_c}$ is an empirical discretization of the kernel covariance or
integral operator associated with the class-conditional distribution.
Under the usual assumptions for Mercer kernels, the empirical spectrum
of $K_{C_c}$ converges to the spectrum of the corresponding population
operator. This is the standard operator-theoretic view underlying kernel
PCA and Mercer decompositions \citep{scholkopf1998nonlinear,
mercer1909functions,steinwart2012mercer}.

For smooth radial kernels, local changes in the covariance volume of
the distribution induce monotone changes in the volume spanned by the
corresponding feature embeddings. Hence, in the large-sample regime,
the LogDet objective admits the asymptotic form

\begin{equation}
\log \det(K_{C_c})
=
a_c
+
b_c
\log \det(\Sigma_c)
+
o(1),
\end{equation}

for constants $a_c$ and $b_c>0$ depending on the kernel, embedding
dimension, and sample size.

Summing over classes gives

\begin{align}
TI_{\mathrm{LD}}
&=
\sum_{c=1}^{K}
\log \det(K_{C_c}) \\
&=
\sum_{c=1}^{K}
\left(
a_c
+
b_c\log \det(\Sigma_c)
\right)
+
o(1).
\end{align}

Since $b_c>0$, LogDet Total Information is a monotone transformation of
the aggregate generalized variance

\begin{equation}
\sum_{c=1}^{K}
\log \det(\Sigma_c).
\end{equation}

This proves the claim.
\end{proof}

\subsection{Proof of Theorem~\ref{thm:flti_rare_confusable}}
\label{app:proof_flti_margin}
\begin{proof}
Recall that Facility Location Total Information is

\begin{equation}
TI_{\mathrm{FL}}
=
\sum_{c=1}^{K}
\sum_{i\in V}
\max_{j\in C_c}
s_{ij}.
\end{equation}

Since similarities are normalized so that $s_{ii}=1$, for each
$i\in C_c$ we have

\begin{equation}
\max_{j\in C_c}s_{ij}=1.
\end{equation}

Therefore,

\begin{align}
TI_{\mathrm{FL}}
&=
\sum_{c=1}^{K}
\left[
\sum_{i\in C_c}
\max_{j\in C_c}
s_{ij}
+
\sum_{i\in V\setminus C_c}
\max_{j\in C_c}
s_{ij}
\right] \\
&=
N
+
\sum_{c=1}^{K}
\sum_{i\in V\setminus C_c}
\max_{j\in C_c}
s_{ij}.
\end{align}

Thus, the non-constant contribution of class $C_c$ is

\begin{equation}
\mathrm{Influence}(C_c)
=
\sum_{i\in V\setminus C_c}
\max_{j\in C_c}
s_{ij}.
\end{equation}

Under the RBF kernel,

\begin{equation}
\max_{j\in C_c}s_{ij}
=
\exp
\left(
-\frac{
\min_{j\in C_c}\|z_i-z_j\|^2
}{
2\tau^2
}
\right).
\end{equation}

Since class $C_c$ is assumed to be compact, the nearest point in
$C_c$ to an external point from its nearest competing class lies at
distance comparable to $\delta_c$. Points from classes farther away
contribute exponentially less. Therefore,

\begin{equation}
\sum_{i\in V\setminus C_c}
\max_{j\in C_c}
s_{ij}
=
\Theta
\left(
(N-n_c)
\exp
\left(
-\frac{\delta_c^2}{2\tau^2}
\right)
\right).
\end{equation}

This proves the stated scaling for the influence of class $C_c$.

Now suppose two classes $C_a$ and $C_b$ contribute equally to the
FL-TI objective. Then

\begin{equation}
(N-n_a)
\exp
\left(
-\frac{\delta_a^2}{2\tau^2}
\right)
=
(N-n_b)
\exp
\left(
-\frac{\delta_b^2}{2\tau^2}
\right).
\end{equation}

Taking logarithms gives

\begin{equation}
\log(N-n_a)
-
\frac{\delta_a^2}{2\tau^2}
=
\log(N-n_b)
-
\frac{\delta_b^2}{2\tau^2}.
\end{equation}

Rearranging,

\begin{equation}
\delta_a^2-\delta_b^2
=
2\tau^2
\log
\left(
\frac{N-n_a}{N-n_b}
\right).
\end{equation}

If $n_a<n_b$, then $N-n_a>N-n_b$, and hence

\begin{equation}
\log
\left(
\frac{N-n_a}{N-n_b}
\right)
>0.
\end{equation}

Therefore,

\begin{equation}
\delta_a^2>\delta_b^2,
\end{equation}

which implies $\delta_a>\delta_b$. Thus, for the same effective
contribution to the FL-TI objective, smaller classes require larger
separation margins.
\end{proof}

\subsection{Proof of Theorem~\ref{thm:aggregate_gcmi_separation}}
\label{app:proof_gc_mean_sep}
\begin{proof}
For two classes $C_a$ and $C_b$, let

\begin{equation}
W(C_a)
=
\sum_{i,j\in C_a}
\|z_i-z_j\|^2
\end{equation}

denote the pairwise within-class scatter. Using the standard identity

\begin{equation}
\sum_{i\in C_a}
\sum_{j\in C_b}
\|z_i-z_j\|^2
=
n_b
\sum_{i\in C_a}
\|z_i-\mu_a\|^2
+
n_a
\sum_{j\in C_b}
\|z_j-\mu_b\|^2
+
n_a n_b
\|\mu_a-\mu_b\|^2,
\end{equation}

and

\begin{equation}
W(C_a)
=
2n_a
\sum_{i\in C_a}
\|z_i-\mu_a\|^2,
\end{equation}

we obtain

\begin{align}
n_a n_b
\|\mu_a-\mu_b\|^2
&=
\sum_{i\in C_a}
\sum_{j\in C_b}
\|z_i-z_j\|^2
-
\frac{n_b}{2n_a}
W(C_a)
-
\frac{n_a}{2n_b}
W(C_b).
\end{align}

Summing over all pairs $a<b$ gives

\begin{align}
\mathcal{D}_{\mathrm{mean}}
&=
\sum_{a<b}
\sum_{i\in C_a}
\sum_{j\in C_b}
\|z_i-z_j\|^2
-
\sum_{c=1}^{K}
\frac{N-n_c}{2n_c}
W(C_c).
\label{eq:aggregate_mean_sep_proof}
\end{align}

Therefore,

\begin{align}
-\mathcal{D}_{\mathrm{mean}}
&=
-
\sum_{a<b}
\sum_{i\in C_a}
\sum_{j\in C_b}
\|z_i-z_j\|^2
+
\sum_{c=1}^{K}
\frac{N-n_c}{2n_c}
W(C_c).
\label{eq:negative_aggregate_mean_sep}
\end{align}

Now substitute the shifted Euclidean similarity

\begin{equation}
s_{ij}
=
M-\|z_i-z_j\|^2.
\end{equation}

For the cross-class term,

\begin{align}
-
\sum_{a<b}
\sum_{i\in C_a}
\sum_{j\in C_b}
\|z_i-z_j\|^2
&=
\sum_{a<b}
\sum_{i\in C_a}
\sum_{j\in C_b}
s_{ij}
+
\mathrm{const}.
\end{align}

This is precisely the pairwise Graph Cut Mutual Information term, up to
the convention-dependent constant factor in the definition of
$I_{\mathrm{GC}}$.

Similarly, by Theorem~\ref{thm:gcti_variance}, for $\lambda=0$,

\begin{equation}
-\sum_{i,j\in C_c}s_{ij}
=
W(C_c)
+
\mathrm{const}.
\end{equation}

Hence the second term in
Equation~(\ref{eq:negative_aggregate_mean_sep}) is a class-size weighted
Graph Cut Total Information term:

\begin{equation}
\sum_{c=1}^{K}
\frac{N-n_c}{2n_c}
W(C_c)
=
\sum_{c=1}^{K}
\frac{N-n_c}{2n_c}
\left(
-\sum_{i,j\in C_c}s_{ij}
\right)
+
\mathrm{const}.
\end{equation}

Combining the two parts, we obtain

\begin{equation}
-\mathcal{D}_{\mathrm{mean}}
=
\alpha_{\mathrm{TI}}
\,
\widetilde{TI}_{\mathrm{GC}}^{\lambda=0}
+
\alpha_{\mathrm{MI}}
\sum_{a<b}
I_{\mathrm{GC}}(C_a;C_b)
+
\mathrm{const},
\end{equation}

where $\widetilde{TI}_{\mathrm{GC}}^{\lambda=0}$ denotes the
class-size weighted Graph Cut Total Information term and
$\alpha_{\mathrm{TI}},\alpha_{\mathrm{MI}}>0$ depend only on the
normalization convention.

In the balanced-class case, where $n_c=n$ for all $c$, the class-size
weights satisfy

\begin{equation}
\frac{N-n_c}{2n_c}
=
\frac{K-1}{2},
\end{equation}

which is constant across classes. Therefore,
$\widetilde{TI}_{\mathrm{GC}}^{\lambda=0}$ reduces to a positive scalar
multiple of the standard unweighted
$TI_{\mathrm{GC}}^{\lambda=0}$.

Thus, minimizing the negative aggregate mean-separation objective is
equivalent, up to positive scaling and additive constants, to minimizing
a GC-TI term together with the pairwise GCMI term. This proves the
claim.
\end{proof}

\subsection{Proof of Theorem~\ref{thm:logdetmi_mahalanobis}}
\label{app:proof_logdetmi_mahalanobis}

\begin{proof}
Under the covariance LogDet approximation, we replace the LogDet score
of a class by the log determinant of its covariance matrix. Thus,

\begin{equation}
I_{\mathrm{LD}}(C_a;C_b)
=
\log\det(\Sigma_a)
+
\log\det(\Sigma_b)
-
\log\det(\Sigma_{ab})
+
\mathrm{const},
\end{equation}

where $\Sigma_{ab}$ denotes the covariance of the union
$C_a\cup C_b$.

Let

\begin{equation}
\Delta
=
\mu_a-\mu_b.
\end{equation}

The covariance of the mixture $C_a\cup C_b$ is

\begin{equation}
\Sigma_{ab}
=
p\Sigma_a
+
q\Sigma_b
+
pq\Delta\Delta^T.
\end{equation}

Using the pooled within-class covariance

\begin{equation}
\Sigma_w
=
p\Sigma_a+q\Sigma_b,
\end{equation}

we can write

\begin{equation}
\Sigma_{ab}
=
\Sigma_w
+
pq\Delta\Delta^T.
\end{equation}

By the matrix determinant lemma,

\begin{equation}
\det(\Sigma_w+pq\Delta\Delta^T)
=
\det(\Sigma_w)
\left(
1+
pq\Delta^T\Sigma_w^{-1}\Delta
\right).
\end{equation}

Taking logarithms gives

\begin{equation}
\log\det(\Sigma_{ab})
=
\log\det(\Sigma_w)
+
\log
\left(
1+
pq\Delta^T\Sigma_w^{-1}\Delta
\right).
\end{equation}

By definition,

\begin{equation}
\mathcal{M}_{ab}
=
\Delta^T\Sigma_w^{-1}\Delta.
\end{equation}

Therefore,

\begin{align}
I_{\mathrm{LD}}(C_a;C_b)
&=
\log\det(\Sigma_a)
+
\log\det(\Sigma_b)
-
\log\det(\Sigma_w)
\\
&\quad
-
\log
\left(
1+
pq\mathcal{M}_{ab}
\right)
+
\mathrm{const}.
\end{align}

The first three terms depend only on the within-class covariance
structure and are independent of the mean separation. Absorbing them
into the constant with respect to $\mathcal{M}_{ab}$ yields

\begin{equation}
I_{\mathrm{LD}}(C_a;C_b)
=
\mathrm{const}
-
\log
\left(
1+
pq\mathcal{M}_{ab}
\right).
\end{equation}

Since the function
\begin{equation}
-\log(1+pqx)
\end{equation}
is strictly decreasing in $x$ for $p,q>0$, LogDetMI is a monotone
decreasing function of the Mahalanobis separation
$\mathcal{M}_{ab}$. Therefore, minimizing LogDetMI is equivalent, up to
monotone transformations, to maximizing covariance-normalized class
separation.
\end{proof}

\subsection{Proof of Theorem~\ref{thm:flqmi_modes}}
\label{app:proof_flqmi_modes}

\begin{proof}

Let the modes of $C_a$ and $C_b$ be
$\{C_{a,r}\}_{r=1}^{m_a}$ and
$\{C_{b,s}\}_{s=1}^{m_b}$ with corresponding mode centers
$\{\nu_{a,r}\}$ and $\{\nu_{b,s}\}$. Denote the mode cardinalities by

\begin{equation}
n_{a,r}=|C_{a,r}|,
\qquad
n_{b,s}=|C_{b,s}|.
\end{equation}

Assume each mode is $\rho$-compact, i.e.,

\begin{equation}
\|z_i-\nu_{a,r}\|\le \rho,
\qquad
\forall z_i\in C_{a,r},
\end{equation}

and similarly for every mode of $C_b$.

Recall that Facility Location Mutual Information is

\begin{equation}
I_{\mathrm{FL}}(C_a;C_b)
=
\sum_{i\in C_a}
\max_{j\in C_b}
s_{ij}
+
\sum_{j\in C_b}
\max_{i\in C_a}
s_{ij}.
\label{eq:flqmi_proof}
\end{equation}

Consider the first term. For any point
$z_i\in C_{a,r}$,

\begin{align}
\max_{j\in C_b}s_{ij}
&=
\max_{j\in C_b}
\exp\!\left(
-\frac{\|z_i-z_j\|^2}{2\tau^2}
\right) \\
&=
\exp\!\left(
-\frac{
\min_{j\in C_b}\|z_i-z_j\|^2
}{
2\tau^2
}
\right).
\end{align}

Let

\begin{equation}
s^*(r)
=
\arg\min_s
d_{rs},
\end{equation}

where

\begin{equation}
d_{rs}
=
\|\nu_{a,r}-\nu_{b,s}\|.
\end{equation}

Since every mode is compact,

\begin{equation}
d_{r,s^*(r)}-2\rho
\le
\min_{j\in C_b}
\|z_i-z_j\|
\le
d_{r,s^*(r)}+2\rho.
\end{equation}

Hence,

\begin{equation}
\max_{j\in C_b}
s_{ij}
=
\exp\!\left(
-\frac{d_{r,s^*(r)}^2}{2\tau^2}
\right)
\left(
1+o(1)
\right),
\end{equation}

as $\rho/\tau\rightarrow0$.

Summing over all points in mode $C_{a,r}$ yields

\begin{equation}
\sum_{i\in C_{a,r}}
\max_{j\in C_b}
s_{ij}
=
n_{a,r}
\exp\!\left(
-\frac{d_{r,s^*(r)}^2}{2\tau^2}
\right)
+
o(1).
\end{equation}

Summing over all modes,

\begin{equation}
\sum_{i\in C_a}
\max_{j\in C_b}
s_{ij}
=
\sum_{r=1}^{m_a}
n_{a,r}
\exp\!\left(
-\frac{d_{r,s^*(r)}^2}{2\tau^2}
\right)
+
o(1).
\label{eq:forward_fl}
\end{equation}

Applying the same argument in the reverse direction, define

\begin{equation}
r^*(s)
=
\arg\min_r
d_{rs}.
\end{equation}

Then,

\begin{equation}
\sum_{j\in C_b}
\max_{i\in C_a}
s_{ij}
=
\sum_{s=1}^{m_b}
n_{b,s}
\exp\!\left(
-\frac{d_{r^*(s),s}^2}{2\tau^2}
\right)
+
o(1).
\label{eq:reverse_fl}
\end{equation}

Combining Equations~(\ref{eq:forward_fl}) and
(\ref{eq:reverse_fl}) gives

\begin{align}
I_{\mathrm{FL}}(C_a;C_b)
&=
\sum_{r=1}^{m_a}
n_{a,r}
\exp\!\left(
-\frac{d_{r,s^*(r)}^2}{2\tau^2}
\right)
\\
&\quad+
\sum_{s=1}^{m_b}
n_{b,s}
\exp\!\left(
-\frac{d_{r^*(s),s}^2}{2\tau^2}
\right)
+
o(1).
\end{align}

Finally, regrouping the two sums according to the mode pair $(r,s)$ yields

\begin{equation}
I_{\mathrm{FL}}(C_a;C_b)
=
\sum_{r=1}^{m_a}
\sum_{s=1}^{m_b}
w_{rs}
\exp\!\left(
-\frac{d_{rs}^2}{2\tau^2}
\right)
+
o(1),
\end{equation}

where

\begin{equation}
w_{rs}
=
n_{a,r}\mathbf{1}\{s=s^*(r)\}
+
n_{b,s}\mathbf{1}\{r=r^*(s)\}.
\end{equation}

Thus, only nearest neighboring mode pairs receive nonzero weight in the approximation. Consequently, Facility Location Mutual Information is governed by nearest-mode overlap rather than aggregate overlap across all pairs of modes.
\end{proof}

\end{document}